\documentclass{article}
\PassOptionsToPackage{numbers, compress}{natbib}

\usepackage[preprint]{neurips_2026}

\usepackage{microtype}
\usepackage{graphicx}
\setkeys{Gin}{width=0.9\linewidth} 
\usepackage{subcaption}
\usepackage{booktabs} 

\usepackage{hyperref}

\usepackage{tabularx}

\usepackage{comment}
\usepackage{amsmath}
\usepackage{amssymb}
\usepackage{mathtools}
\usepackage{amsthm}
\usepackage{xcolor}         
\usepackage{lipsum}

\usepackage{titlesec}
\usepackage[titletoc,toc,title]{appendix}

\usepackage[capitalize,noabbrev]{cleveref}

\theoremstyle{plain}
\usepackage{amsthm, mdframed, xcolor}   

\definecolor{boxrule}{RGB}{136,136,136}
\definecolor{boxbg}{RGB}{247,247,247}

\newmdtheoremenv[
  linecolor=boxrule,        
  backgroundcolor=boxbg,   
  linewidth=0pt,
  topline=false, rightline=false, bottomline=false,
  innerleftmargin=10pt, innerrightmargin=10pt,
  innertopmargin=6pt, innerbottommargin=6pt,
]{theorem}{Theorem}

\newmdtheoremenv[
  linecolor=boxrule,
  backgroundcolor=boxbg,
  linewidth=0pt,
  topline=false, rightline=false, bottomline=false,
  innerleftmargin=10pt, innerrightmargin=10pt,
  innertopmargin=6pt, innerbottommargin=6pt,
]{proposition}{Proposition}

\theoremstyle{definition}

\theoremstyle{remark}

\usepackage[textsize=tiny]{todonotes}

\usepackage[shortlabels]{enumitem}
\usepackage{autobreak}
	
\newcommand{\indep}{\perp \!\!\! \perp}
\usepackage[dvipsnames]{xcolor}
\usepackage{dsfont}

\usepackage[utf8]{inputenc} 
\usepackage[T1]{fontenc}    
\usepackage{hyperref}       
\usepackage{url}            
\usepackage{booktabs}       
\usepackage{amsfonts}       
\usepackage{nicefrac}       
\usepackage{microtype}      

\usepackage{wrapfig}
\title{Diff-CA: Separating Common and Salient Factors with Diffusion Models}

\author{%
 Michaël Soumm \\
INRIA at Univ. Grenoble Alpes\\
  \texttt{michael.soumm@inria.fr} \\
   \And
    Alexandre Fournier Montgieux \\
   CEA List, Palaiseau \\
 \texttt{alexandre.fourniermontgieux@cea.fr} \\
   \AND
   Yunlong He\\
Télécom Paris, Institut Polytechnique de Paris \\
   \texttt{yunlong.he@telecom-paris.fr} \\
   \And
   Pietro Gori \\
   Télécom Paris, Institut Polytechnique de Paris \\
   \texttt{pietro.gori@telecom-paris.fr} \\
   \And
   Alasdair Newson  \\
   Télécom Paris, Institut Polytechnique de Paris \\
  \texttt{alasdair.newson@telecom-paris.fr} \\
}
\usepackage{ragged2e}

\begin{document}

\maketitle

\begin{abstract}

Contrastive Analysis aims to separate factors that are common between two data distributions from those that are salient to only one of them. 
Existing contrastive methods are based on generative models (e.g., VAEs or GANs) that often suffer from limited reconstruction and image quality, which hampers effective latent factor separation and limits their applicability to high-fidelity image generation and edition.
We propose a novel conditioning framework for diffusion models that enables contrastive decomposition without compromising generation quality. We first train a prompt-free, image-conditioned diffusion model, and then learn to decompose the conditioning into a common and a salient factor, using weak supervision. We prove that the additive contrastive factorization, commonly assumed in prior work, is identifiable under mild conditions. This factorization enables targeted operations by swapping or interpolating only the salient factor.
\end{abstract}


\section{Introduction}

Learning common and salient information between data distributions is an important task in many domains of representation learning, such as multi-view \cite{tian2020contrastive,federici2020multi,wang2016deep,richard_shared_2021,dufumier_integrating_2023}, multi-modal \cite{dufumier2025what} representation learning, domain adaptation \cite{lee2021dranet}, subgroups discovery \cite{louiset_ucsl_2021,louiset_automatic_2026} and disentanglement \cite{sanchez2020disentangled}. In this article, we focus on 
\textbf{Contrastive Analysis (CA)}, which is the problem of separating what is \emph{common} to two data distributions from what is \emph{salient} to one of them.
In this setting, we observe two unpaired image distributions: a \emph{background} distribution $p_{X|Y=0}$ and a \emph{target}  distribution $p_{X|Y=1}$, where $X$ represents an image and $Y\in\{0,1\}$ indicates which distribution it belongs to (\textit{i.e.,} a weak binary signal).
The target distribution contains the additional, salient factor, which is absent in the background.




\begin{figure}[t]
    \centering
     \includegraphics[width=\linewidth]{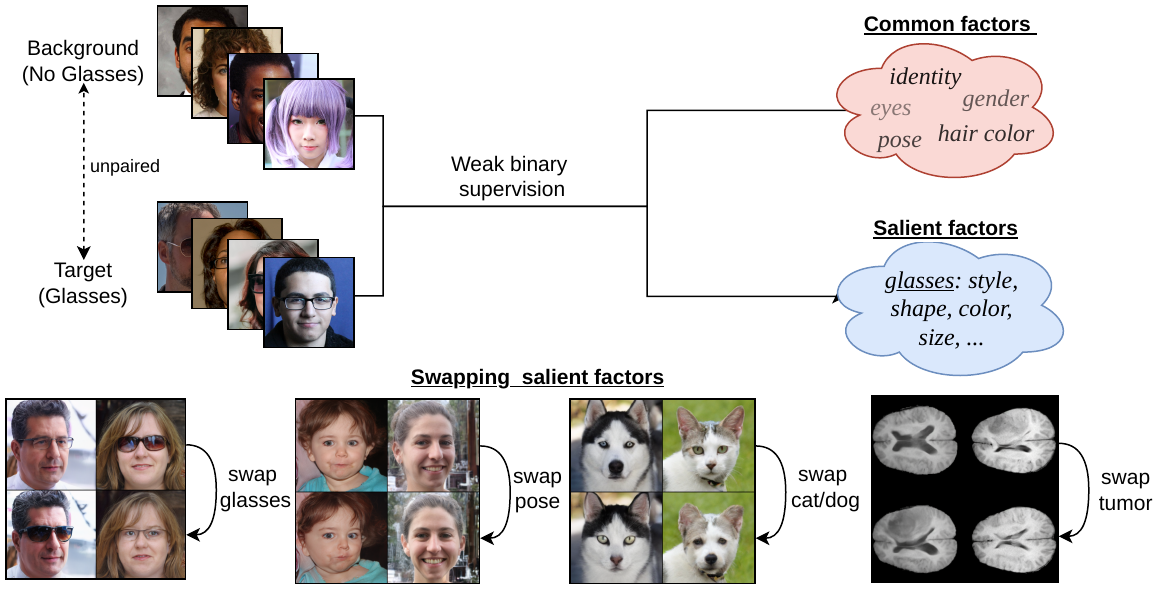}
    \caption{Contrastive Analysis separates \textbf{common} from \textbf{salient} factors between two unpaired data distributions using only weak binary supervision (\textit{i.e.,} dataset-level). The learned common and salient latent spaces can be manipulated for editing or knowledge discovery, for instance.}
    \label{fig:teaser}
\end{figure}

Using as example Fig. \ref{fig:teaser}, where the two datasets contain face images without glasses ($Y=0$) and with ($Y=1$) glasses, the goal of CA is to estimate a representation that decomposes images into common $C$ (\textit{e.g.,} identity, gender) and salient factors $S$ (\textit{e.g.,} glasses style and color). Other examples include medical images with or without a pathology or tumor, and images of products with or without defects, such as metallurgical anomalies (e.g. cracks, fissures) in industrial inspection. Binary group membership (healthy/pathological or pass/fail) are labels that are routinely assigned in hospitals or in industry. However, they might be obtained through intrusive tests (e.g., biopsy, blood tests) and/or be based on complex, salient patterns that can be difficult to capture using text prompts or that are not necessarily known in advance. 
The identification of such patterns could enable diagnosis via non-invasive imaging and a better understanding of the disease. 
For all these applications, CA provides a principled framework for automatically discovering, isolating, and modeling the salient $S$ and common $C$ factors, enabling targeted operations (swapping, interpolation, synthetic/counterfactual generation, or anomaly quantification) without fine-grained annotations or text prompts.

In this work, \textbf{we propose a Contrastive Analysis approach based on latent diffusion models}.
Our method applies CA directly to the conditioning token sequence $Z$ of a latent diffusion architecture.
We learn a decomposition $Z \mapsto (\hat{Z}_C, \hat{Z}_S)$ such that $\hat{Z}_C$ captures the common factor shared across distributions ($C$), while $\hat{Z}_S$ isolates the salient factor specific to the target distribution ($S$).
This yields an effective CA-aligned intervention mechanism: swap, remove, or scale $\hat{Z}_S$ while keeping $\hat{Z}_C$ fixed. 
CA thus provides an alternative to text-driven diffusion editors, which often rely on prompt/inversion/attention heuristics
\cite{hertzprompt,brooks2023instructpix2pix,mokady2023null,parmar2023zero} that can be sensitive to prompt/guidance choices, introduce unintended changes
, 
or that cannot be applied when salient patterns are difficult to describe or are unknown.
We further discuss the limitations of these editors in \autoref{sec:LLMs_failure}.

Furthermore, we provide a theoretical analysis of the common-salient decomposition based on the additive decomposition: $Z=\hat{Z}_C+\hat{Z}_S$, which is commonly assumed in CA \cite{abid2019contrastive,carton2024double,louiset2024sepvae,louiset2024separating,weinberger2022moment,he_learning_2025}. We also prove its identifiability under mild conditions, showing that the additive decomposition in the token space enables the recovery of the true underlying common and salient factors.

Our contributions are as follows: 
\begin{itemize}[leftmargin=*,topsep=2pt,itemsep=1pt]
    \item \textbf{Theoretical Analysis for CA}: Section \ref{sec:theory} provides a rigorous theoretical framework for Contrastive Analysis in the context of generative models. We derive a lower bound on controllability based on mutual information (Prop.~\ref{prop:control}) and prove that under an additive structural assumption, the true common and salient factors are uniquely identifiable from weak supervision (Thm.~\ref{thm:identifiability}).
    \item \textbf{High-fidelity latent space for diffusion models}:
     We construct in Sec.~\ref{subsec:conditioning_pipeline} a conditioning latent space for diffusion models to match the structural assumptions of Sec~\ref{sec:theory}. By training a Cross-Query encoder with DINOv3 features and a novel Color Token, we obtain a latent space suitable for both {high fidelity reconstruction} and {generated image attribute manipulation}.
    \item \textbf{CA on diffusion conditioning tokens}: In Sec~\ref{subsec:separator_training}, we introduce our separator architecture, Diff-CA, and translate the theoretical constraints into practical training losses applied to the token latent space. As evaluated in Sec.~\ref{sec:experiments} across multiple domains, Diff-CA isolates and edits salient attributes while effectively preserving common factors (e.g., identity, gender), achieving state-of-the-art reconstruction and swapping performance compared to existing CA baselines.
\end{itemize}


\section{Related Works}
\label{sec:related_work}
\paragraph{Contrastive analysis.}
Classic CA formulations include contrastive PCA \cite{abid2018exploring}, and VAE based architectures with different regularizations to enforce a common/salient latent separation \cite{abid2019contrastive,weinberger2022moment,louiset2024sepvae,benaim2019domain,bousmalis2016domain,sanchez2020disentangled,kleinman2023gacskorner}. More recent CA work explores adversarial generators to improve sample quality while retaining a common/salient decomposition \cite{gonzalez2018image,carton2024double}, or representation-level separation without relying on high-fidelity generation \cite{louiset2024separating}. A recurring limitation is that imperfect reconstruction or synthesis prevents latent factors from being effectively separated. 

\noindent \textbf{Diffusion conditioning as a controllable interface.}
Diffusion models provide high-fidelity synthesis \cite{ho2020denoising}, and latent diffusion allows for external conditioning through cross-attention over \emph{conditioning embeddings/tokens} \cite{rombach2022high}. Several works expand conditioning pathways to improve control, e.g. by adding dedicated conditioning branches or image-derived prompts \cite{zhang2023adding,ye2023ip}. In parallel, other lines of work aim to expose more structured latent variables for diffusion to enable semantic manipulation \cite{preechakul2022diffusion,kwondiffusion,park2023understanding,jeong2024training}. Our work is complementary: instead of proposing a new editing heuristic or a new global latent, we bring CA to diffusion by decomposing the image-conditioning tokens themselves into common $C$ and salient $S$ under the target/background weak supervision signal.

\noindent \textbf{Diffusion image editing.}
A large body of diffusion editing methods performs user-specified edits via prompts, attention control, and/or inversion \cite{hertzprompt,mokady2023null,brooks2023instructpix2pix,ruiz2023dreambooth}. These approaches can produce impressive edits, but they are not designed to \emph{identify} and \emph{separate} dataset-level common and salient generative factors between two distributions, under weak supervision.  Instead, they steer generation through prompt- or inversion-dependent mechanisms. 

\noindent \textbf{Disentanglement.} Conceptually, our goal is also related to disentanglement, which aims to produce changes of semantically meaningful attributes (e.g., gender) by altering a \textit{single} latent component. Disentanglement is known to be ill-posed without additional constraint or supervisions \cite{higgins2017beta,chen2016infogan,locatello2019challenging} and its objective is complementary to the one of CA. Indeed, CA deals with the latent separation of the generative factors $C$ and $S$, which may be further disentangled (\textit{i.e.,} one factor per attribute),
using prior information about salient/common attributes (left as future work in this article). 


\section{Theoretical Analysis of CA decomposition}
\label{sec:theory}

\begin{wrapfigure}[17]{R}{0.5\textwidth}
\vspace*{-0.5cm}
    \centering
\includegraphics[width=0.49\textwidth,trim={0.15cm 0 0.5cm 0},clip]{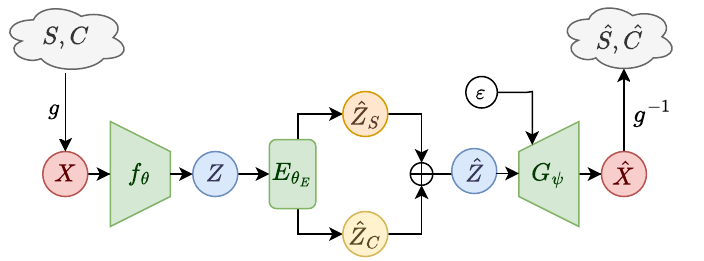}
    \caption{Conceptual view of our decomposition. Factors {\color{gray} $(S, C)$} generate an observed image {\color{BrickRed} $X$}. We learn a feature extractor {\color{ForestGreen} $f_\theta$} that projects the images into latents {\color{NavyBlue} $Z$}, and an encoder  {\color{ForestGreen} $E_{\theta_E}$} that splits  {\color{NavyBlue} $Z$},  into {\color{BurntOrange} ($\hat{Z}_S, \hat{Z}_C$)} aimed at representing {\color{gray} $(S, C)$}. {\color{BurntOrange} ($\hat{Z}_S, \hat{Z}_C$)} sum to  {\color{NavyBlue} $\hat{Z}$}, which conditions a diffusion model {\color{ForestGreen} $G_\psi$} aimed at generating an approximation {\color{BrickRed} $\hat{X}$} of {\color{BrickRed} $X$}.} 
    \label{fig:factors}
\end{wrapfigure}

In this section, we introduce the additive structural assumptions of the CA framework. Furthermore, we present \textbf{theoretical analysis} to show that these assumptions lead to an \textbf{identifiable CA decomposition}. 

We consider image-label pairs $(x,y)\sim p_{X, Y}$ in $\mathcal{X}\times\{0,1\} $, where $y$ is a binary indicator of the absence~/~presence of the target attributes. We refer to samples from $p_{X| Y=0}$ and $p_{X|Y=1}$ as negative (or \textit{background}) and positive (or \textit{target}) samples, respectively. We assume that $X$ is generated by an unknown, deterministic and invertible process $g$ applied to two \textit{unobserved} and \textit{independent} generative factors $S$ and $C$, such that $g:\mathcal{S}~\times~\mathcal{C}~\to~\mathcal{X}$, where $S\in\mathcal{S}\subset\mathbb R^{d_S}$ is the \emph{salient} factor and $C\in\mathcal{C}\subset\mathbb R^{d_C}$ the \emph{common} one:
\begin{align}
\label{eq:generative_process}
X = \begin{cases}
        g(\ 0\ , C) & \text{if } Y=0 \\
        g(\smash{\underbrace{S}_{\neq0}},C) & \text{if } Y=1,
    \end{cases} \quad \text{and}\quad   C \indep S.
\end{align}
\vspace{1mm}
\hspace{-1mm}

The factor $S$ encodes the target attributes, describing their characteristics such as shape, texture, or style. Since $C$ should fully represent background samples ($Y=0$), we model the \textit{absence} of salient information in $X$ by fixing $S=0$, as is commonly done in CA \cite{abid2018exploring,abid2019contrastive,carton2024double,louiset2024sepvae,louiset2024separating,weinberger2022moment}.

Our goal is thus to learn, from images $X$ and weak binary labels $Y$ alone, \textbf{a representation $(\hat{Z}_S, \hat{Z}_C)$ that mirrors  the $(C,S)$ decomposition}, where $\hat{Z}_S$ models all of the variability of the salient factor $S$ (not simply its absence~/~presence), while $\hat{Z}_C$ models the common factors $C$.


\subsection{Additive Structure of the Latent Space and Identifiability}
\label{sec:structural_assumption}
\paragraph{Additive decomposition.} Before aiming for a decomposition of salient and common information, we need to learn a conditioning representation ${Z= f_\theta(X) \in \mathbb{R}^d}$ that accurately captures both common and salient factors. 
To make the decomposition of $S$ and $C$ tractable, we postulate that the encoder $f_\theta$ maps the complex non-linear interactions of the pixel space into a structured \emph{linear additive latent space}, inspired by previous work on CA \cite{weinberger2022moment, louiset2024sepvae,louiset2024separating, carton2024double, locatello2019challenging}. 
Formally, we assume the existence of mappings $\varphi_S: \mathcal{S} \to \mathbb{R}^d$ and $\varphi_C: \mathcal{C} \to \mathbb{R}^d$ such that the conditioning $Z$ decomposes as 
\begin{equation}
    \text{Im}(f_\theta) = \text{Im}(\varphi_S) \oplus \text{Im}(\varphi_C) \quad \text{  i.e. }  Z = Z_S + Z_C,
   \label{eq:additive_struct}
\end{equation}
where $\oplus$ denotes a direct sum and we define the \emph{encoded factors} as $Z_S := \varphi_S(S)$ and $Z_C := \varphi_C(C)$. Since $S \indep C$ by definition, $Z_S$ and $Z_C$ are also statistically independent. While high-level semantic linearity is a strong assumption, we empirically validate in Sec.~\ref{sec:finegrain} that our DINOv3-based token representation supports such vector arithmetic.
Importantly, weak supervision provides  a \textit{pinning constraint}. In the background distribution ($Y=0$), the salient factor is absent ($S=0$). We require the salient mapping to vanish at this reference point:\footnote{This is up to a shift constant: we can simply shift $\varphi_S$ and $\varphi_C$ by a constant to satisfy this requirement.} $S = 0 \implies Z_S = 0$.

\paragraph{Identifiability of the Decomposition.}
\label{sec:identifiability}
The direct sum assumption implies that a unique decomposition of $Z$ exists, but we do not know the subspaces for $Z_S$ and $Z_C$ that span it. Therefore, we cannot simply project $Z$, we must learn these subspaces. To this end, we introduce a separator encoder network to separate the two: $E_{\theta_E}(Z) := (\hat{Z}_S, \hat{Z}_C)$, aimed at inferring $(Z_S, Z_C)$ given $Z$. We refer to $ (\hat{Z}_S, \hat{Z}_C)$ as latent \emph{codes} (learned) to distinguish them from true common/salient \emph{factors} (theoretical).

A key question is under what theoretical conditions we can guarantee the recovery of the true common/salient factors, avoiding spurious decompositions. The following theorem establishes that three structural assumptions are sufficient to ensure identifiability.

\begin{theorem}[Identifiability of Additive Factors]
\label{thm:identifiability}
Assume the conditioning $Z$ is generated by the direct sum of independent factors $Z_S$ and $Z_C$, where $Z_S$ vanishes for background samples ($Y=0$). If a learned decomposition $ Z \mapsto (\hat{Z}_S, \hat{Z}_C)$ satisfies:
\begin{enumerate}[leftmargin=1cm, topsep=0pt]
    \item \textbf{Additive Reconstruction:} The learned codes sum to the input: $\hat{Z}_S + \hat{Z}_C = Z$ 
    \item \textbf{Pinning:} The salient code vanishes only for background samples: $Y=0 \Longleftrightarrow \hat{Z}_S = 0$.
    \item \textbf{Independence:} The learned codes are independent: $\hat{Z}_S \indep \hat{Z}_C$.

\end{enumerate}
Then, the learned codes uniquely identify the true encoded factors:
\begin{equation}
    \hat{Z}_S = Z_S \quad \text{and} \quad \hat{Z}_C = Z_C
\end{equation}
\end{theorem}
The proof is in App.~\ref{app:ident_proof}. This result provides the theoretical justification for our approach. While the true subspaces for $(Z_S, Z_C)$ are unobserved, the theorem guarantees that any decomposition $(\hat{Z}_S, \hat{Z}_C)$ satisfying these three structural constraints will perfectly recover the true underlying factors. We present the practical training objectives used to enforce these constraints in Section \ref{subsec:separator_training}. Importantly, this result requires no assumptions about the intrinsic effective dimension of $\hat{Z}_S$ and $\hat{Z}_C$.

\subsection{Control over Generative Models}
\label{subsec:generative}

To translate the learned latent decomposition into explicit semantic control over pixels, we introduce a conditional generative model $G_\psi$. Given the conditioning representation $Z$, it approximates the data distribution via :
\begin{equation}
    \hat{X} = G_\psi(Z, \varepsilon), \qquad \varepsilon \sim \mathcal{N}(0, I),
    \label{eq:generator}
\end{equation}
where $\varepsilon$ is independent noise. We denote by $\hat{S} \in \mathcal{S}$ and $\hat{C} \in \mathcal{C}$ the latent factors of the generated image $\hat{X}$. 

To ensure that varying the $\hat{Z}_S$ component within $Z$ effectively modifies \textit{only} the salient factor of the generated image (independent of the common information or noise), we 
quantify control via the mutual information $I(\hat{S}; \hat{Z}_S)$. Maximizing this quantity ensures that the latent code $\hat{Z}_S$ carries significant information about the output $\hat{S}$ (and similarly for $\hat{C}$ and $\hat{Z}_C$).

The next result establishes that valid control is guaranteed if the model satisfies three conditions: \textbf{high fidelity} to the real factors $S$ and $C$,  \textbf{low noise dependance} of the generated images and \textbf{low entanglement} between $\hat{Z}_C$ and $\hat{Z}_S$.

\begin{proposition}[Control decomposition]\label{prop:control}Assume a decomposition $(\hat{Z}_S, \hat{Z}_C)$ satisfies the identifiability conditions of Th. \ref{thm:identifiability}, with $Z=\hat{Z_S} + \hat{Z}_C
$ Then, control is lower-bounded by:
\begin{equation}
\label{eq:control-bound-S} I(\hat{S}; \hat{Z}_S) \ge \underbrace{I(S; \hat{S})}_{\text{Fidelity}} - \underbrace{I(\hat{X}; \varepsilon | Z)}_{\text{Noise dependence}} - \underbrace{I(S; \hat{Z}_C | \hat{Z}_S)}_{\text{Entanglement}}
\end{equation}
and, for the common context:
\begin{equation}
\label{eq:control-bound-C} I(\hat{C}; \hat{Z}_C) \ge I(C; \hat{C}) - I(\hat{X}; \varepsilon | Z) - I(C; \hat{Z}_S | \hat{Z}_C)
\end{equation}
\end{proposition}
The proof can be found in App.~\ref{app:control_proof}. By the Data Processing Inequality, $I(\hat{S}; S)\le I(\hat{X}, X)$ (and similarly $I(\hat{C}; C)\le I(\hat{X}, X)$). This means that to effectively use $\hat{Z}_S$ and $\hat{Z}_C$ to control the factors of a generated image, we need $f_\theta$  and $G_\psi$ to act as an autoencoder $X\mapsto\hat{X}$ where: (1) the factors of $X$ need to be preserved in $\hat{X}$ (\textit{i.e.,} high Fidelity), (2) the generated image must minimally depend on the noise $\varepsilon$ (\textit{i.e.,} low Noise Dependance) and (3) $\hat{Z}_S$ and $\hat{Z}_C$ should be independent. 


\section{Methodology}
\label{sec:methodology_main}
To translate the theoretical requirements of Section \ref{sec:theory} into a practical framework, we propose a two-stage training procedure. First, to ensure low Noise Dependance, we must construct the feature extractor $f_\theta$ and generative model $G_\psi$ such that the generated attributes are driven entirely by the conditioning $Z=f_\theta(X)$ rather than the initial sampling noise $\varepsilon$ (Stage 1). While standard diffusion and flow-matching models often rely heavily on the starting noise for structural variability\cite{kwondiffusion, guo2024initno,grimal2025sagalearningsignalaligneddistributions}, we require the common and salient factors of the generated image to be fully actionable through the conditioning alone. Second, once this actionable latent space is established, we train the encoder network $E_{\theta_E}$ to enforce the structural conditions on $\hat{Z}_S$ and $\hat{Z}_C$, ensuring high Fidelity and low Entanglement (Stage 2).

\subsection{Image-Conditioned Generator Architecture}
\label{subsec:conditioning_pipeline}

Standard prompt-based methods inject features into frozen text-to-image backbones, often failing to override the base model's priors and leading to identity loss or style drift (Fig.~\ref{fig:prompt_failure}). Consequently, we train a generative model \emph{from scratch} using image-only conditioning, derived from DINOv3 \cite{simeoni2025dinov3} embeddings, consisting in $K$ tokens. We adopt a Flow Matching objective, which enforces a deterministic mapping between noise and data, ensuring that visual fidelity is entirely determined by the structure of $Z$ rather than stochastic sampling paths. Details on the generator and the Cross-Query (CQ and CQC) encoding of the conditioning Z, which allows low noise dependence for $G_\psi$, can be found in App.~\ref{app:subsec:conditioning_pipeline} and in Fig~\ref{fig:combined_pipeline_and_query} (Left).

\subsection{Separator Encoder Architecture and Training}
\label{subsec:separator_training}

Our CA approach, \textbf{Diff-CA}, leverages the formulation from Sec.~\ref{sec:theory} and the conditioning pipeline from App.~\ref{sec:finegrain}. We implement the encoder $E_{\theta_E}$ as a small 5-block Transformer encoder. A learnable \texttt{CLS} token is prepended to the sequence $Z$. After processing, the CLS token state is projected via a linear layer to form the salient code $\hat{Z}_S$, while the remaining $K$ output tokens form the common information $\hat{Z}_C$ (Fig.~\ref{fig:combined_pipeline_and_query} (Right)).

\textbf{Additive structure.} We enforce the sum-constraint ${Z = \hat{Z}_S+ \hat{Z}_C}$ using the reconstruction loss:
\begin{equation}
    \mathcal{L}_{rec} = \mathbb{E}\left[||Z - (\hat{Z}_S+\hat{Z}_C)||_2^2\right]
\end{equation}

\textbf{Pinning.} To enforce the pinning condition ($Z_S=0$) for the background distribution ($Y=0$), we apply a norm-based regularization on $\hat{Z}_S$. This penalty encourages no (salient) information content for background samples ($Y=0$) while preventing collapse for target samples ($Y=1$):
\begin{equation}
\mathcal{L}_{pin} = \mathbb{E}\left[\mathds{1}_{Y=0} ||\hat{Z}_S||_2^2 + \mathds{1}_{Y=1}e^{-\|\hat{Z}_S||_2^2}\right]
\end{equation}

\paragraph{Cycle Consistency.} To enforce the independence condition, we require that the learned factors be interchangeable between pairs of samples with a cycle-consistency loss. We construct a mixed latent code ${Z}^{mix} = \hat{Z}_C^a + \hat{Z}_S^b$ using factors from two independent samples $a$ and $b$. The encoder is trained to recover the original codes from this mixed sample:
\begin{align}
\mathcal{L}_{cyc} = \mathbb{E}&\left[|| \hat{Z}_C^{mix} - \hat{Z}_C^a ||_2^2 + ||\hat{Z}_S^{mix} - \hat{Z}_S^b ||_2^2]\right]
\end{align}
where $ (\hat{Z}_S^{mix}, \hat{Z}_C^{mix}) = E_{\theta_E}({Z}^{mix})$. This prevents the estimation of the common code from relying on correlations with a specific salient code, effectively pushing the learned distributions toward independence (see App.~\ref{app:cycle}). Intuitively, we want swapped counterfactuals to be "valid" samples ($Z_{mix} \overset{Law}{=} Z$). We rely on the separator encoder $E_{\theta_E}$ itself to act as a weak, implicit discriminator. Specifically, we maintain an exponential moving average (EMA) copy of $E_{\theta_E}$, which is used to re-encode swapped counterfactuals $Z_{\mathrm{mix}}$. We detail this heuristic regularization strategy in App.~\ref{app:separation_details}.

\textbf{Adversarial Training.} To further discourage $\hat{Z}_C$ from being predictive of $Y$ (thus encouraging $\hat{Z}_C$ to encode common information), we employ a Gradient Reversal Layer (GRL) adversarial setup \cite{ganin2015unsupervised}. A discriminator $D_{adv}$ is trained to predict the domain $Y$ from the common codes $\hat{Z}_C$ by minimizing the binary cross-entropy:
\begin{equation}\mathcal{L}_{adv} = - \mathbb{E} \left[ y \log \hat{y} + (1-y) \log (1 - \hat{y}) \right]\end{equation}
where $\hat{y} = D_{adv}(\hat{Z}_C)$. To optimize this adversarial loss without instability, we introduce a self-tuning mechanism that dynamically scales the GRL strength to maintain a target discriminator accuracy. We detail this training protocol in App.~\ref{app:separation_details}.

\textbf{Total Objective.} The total loss is defined as:
\begin{equation}
    \mathcal{L}_{tot} =  \mathcal{L}_{rec} + \lambda_1\mathcal{L}_{pin} +   \lambda_2\mathcal{L}_{cycle} +\lambda_3\mathcal{L}_{adv}
\end{equation}
The optimization is a min-max game where $D_{adv}$ minimizes classification error, while the encoder maximizes it (via GRL with strength $\lambda_{adv}$) to remove $Y$-related information:
\begin{equation}
\min_{\theta_{adv}} \mathcal{L}_{adv} \quad \text{and} \quad \min_{\theta_E} \left( \mathcal{L}_{tot} - \lambda_{adv} \mathcal{L}_{adv} \right)
\end{equation}

\begin{figure}[t]
    \centering
    \begin{minipage}[t]{0.54\linewidth}
        \centering
        \includegraphics[width=0.99\linewidth, trim={1.5cm 0 0.6cm 0}, clip]{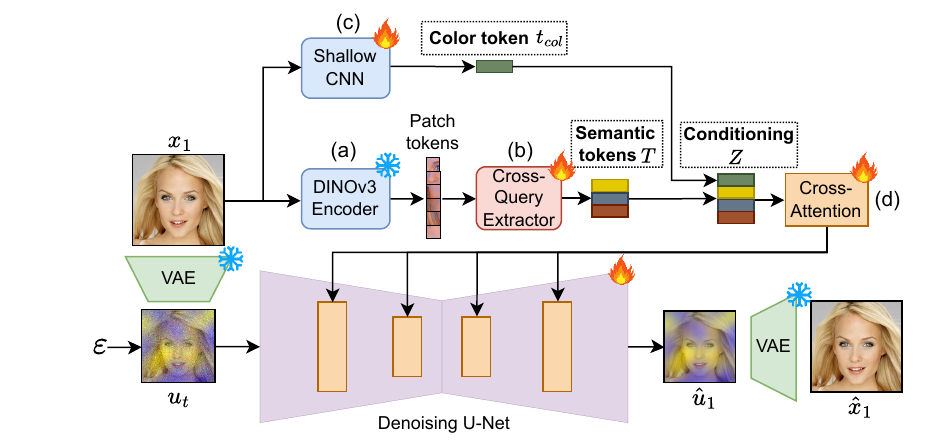}
        \vspace{0.1cm}
    \end{minipage}\hfill
    \begin{minipage}[t]{0.44\linewidth}
        \centering
        \includegraphics[width=0.99\linewidth, trim={0cm 0 1.1cm 0}, clip]{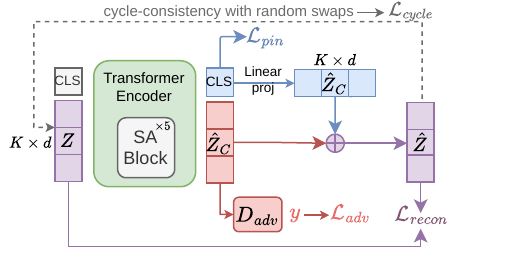}
        \vspace{0.1cm}
    \end{minipage}
    
    \caption{Our  2-Stage training protocol. \textbf{Left:} Conditioning pipeline. An input image $x_1$ is mapped to a latent $u_0$ by a VAE encoder and to DINOv3 features (a). A cross-query module (b) produces semantic tokens $T$, and a small CNN (c) produces a color token $t_{\text{col}}$. The concatenated tokens $Z$ condition a U-Net diffusion model via cross-attention (d). \textbf{Right:} {Training pipeline of our Common-Salient encoder. A {\color{gray} CLS} token is prepended to the {\color{Fuchsia} input tokens}. The encoder separates the {\color{NavyBlue}common} and {\color{BrickRed}salient} parts, whose sum reconstructs the input. An anchor loss $\mathcal{L}_{anchor}$ ensures that the salient part has no information about background samples, while adversarial training purifies the common part. A cycle-consistency protocol ensures that counterfactuals (swaps) produce meaningful outputs.}}
    \label{fig:combined_pipeline_and_query}
\end{figure}

\section{Experiments}
\label{sec:experiments}

\noindent\textbf{Implementation Details.} Comprehensive details regarding architectures, hyperparameters, and data preprocessing for all stages are provided in App. \ref{app:implementation} and \ref{app:separation_details}.

\subsection{Empirical Validation of the Conditioning Space}
\label{subsec:sanity_checks}

\begin{figure}[htbp]
    \centering
    
    \begin{minipage}[t]{0.40\linewidth}
        \centering
        \textbf{Reconstruction}\par\vspace{9pt}
        \resizebox{\textwidth}{!}{
        \begin{tabular}[t]{@{}ccccc@{}}        
        \toprule
        Encoder & Cond. size & SSIM $\uparrow$ & LPIPS $\downarrow$ & FID $\downarrow$  \\
        \midrule
        DiffAE  & 512 & 0.530 & 0.209 & 14.28 \\
        EncDiff  & $32\times 128$ & 0.420 & 0.323 & 18.43 \\
        CLIP & $197\times 128$ & 0.490 & 0.201 & 7.69  \\
        DINOv3 & $261\times 128$ & \textbf{0.645} & \textbf{0.099} & 5.66 \\ \midrule
        \textbf{CQ (ours)}   & $32 \times 128$ & 0.589  & 0.124 & \underline{5.59} \\
        \textbf{CQC (ours)} & $32 \times 128$ & \underline{0.611} & \underline{0.114} & \textbf{5.52} \\
        \bottomrule
        \end{tabular}}
    \end{minipage}\hfill
    \begin{minipage}[t]{0.58\linewidth}
      \centering
        \textbf{Latent Space Smoothness}\par\vspace{3pt}
    \includegraphics[width=0.99\linewidth]{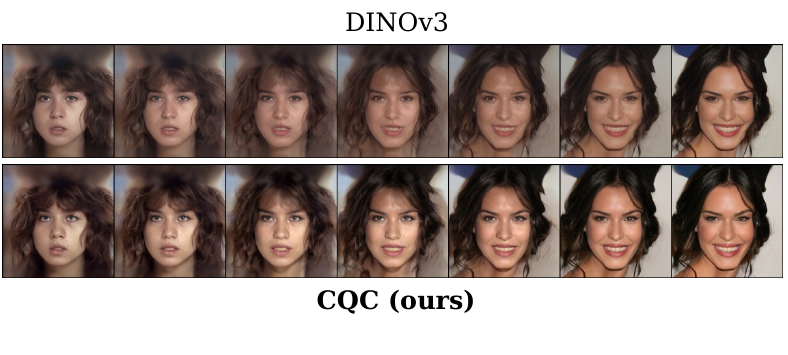}
    \end{minipage}
    
    \caption{\textbf{Left}: Comparison of image reconstruction on CelebA-HQ $256{\times}256$. This serves as a sanity check to make sure that the images produced by our model are of good quality. Our methods (CQ and CQC) present competitive reconstructions compared to DINOv3 with fewer tokens, while also being able to carry out editing and image control (which DINOv3 cannot). \textbf{Right}: Exploring the $Z$-space. We linearly interpolate between $Z$ codes of two images and generate from random noise at each interpolation point. Our method yields smooth changes in identity, attributes, and color without introducing artifacts.}
    \label{fig:recon_and_interp}
\end{figure}
\noindent\textbf{Setup and Baselines.} We evaluate our conditioning architecture against several established baselines: DiffAE~\cite{preechakul2022diffusion}, which utilizes a ResNet to produce a single 512-d conditioning vector; EncDiff~\cite{yang2024diffusion}, which trains a CNN to generate $32 \times 128$ tokens; and projected variants of CLIP and DINOv3. For the latter, we utilize all last-layer activations (197 and 261 tokens) and project each to 128-$d$ before conditioning. All models are trained on FFHQ~\cite{karras2019style} and tested on CelebA-HQ~\cite{liu2015faceattributes}$512\times 512$.

\textbf{Reconstruction.}
To validate the fidelity of our generative model, we evaluate the reconstruction performance of our Cross-Query (CQ) and Color (CQC) encoders (Implementation details in App.~\ref{app:subsec:conditioning_pipeline}) against established conditioning approaches (Fig.~\ref{fig:recon_and_interp}, Left). Our conditioning yields highly competitive reconstruction metrics, despite utilizing a significantly smaller token representation than its strongest baseline, DINOv3. 
Qualitative results can be found in App. ~\ref{app:implementation} and \ref{app:cond_qualitative}

\textbf{Linear structure of the conditioning space.}
To evaluate the geometric properties of our learned token manifold, we also perform linear interpolation between the conditioning codes ($Z$) of two distinct images (Fig.~\ref{fig:recon_and_interp}, Right). When compared with DINOv3, our conditioning yields drastically smoother, artifact-free transitions, indicating that our conditioning provides a more \textbf{continuous} and \textbf{well-structured} representation while maintaining \textbf{competitive reconstruction} performance.

\subsection{Contrastive Decomposition and Swapping}
\label{subsec:main_results}

\noindent \textbf{Datasets and Attributes.} We evaluate our method across three distinct domains to demonstrate its generality. We use \textbf{FFHQ} \cite{karras2019style} and target two salient attributes: Glasses (binary: No Glasses vs. Glasses), Smile (Neutral vs. Smiling). For evaluation only, we use fine-grained labels when possible (Reading vs. Sunglasses) to measure latent structure discovery. On \textbf{AFHQ} \cite{choi2020starganv2}, we treat species as the salient factor, using a binary Cat vs. Dog setup. \textbf{BraTS 2023} \cite{kazerooni2024brain} is a medical imaging benchmark composed of brain MRI data where the salient factor $S$ is the presence of a brain tumor. We use 2D slice images with a slice-level binary (tumor) label as $Y$.

\noindent\textbf{Baselines.} We compare our method against three state-of-the-art CA generative models: MM-cVAE \cite{weinberger2022moment}, SepVAE \cite{louiset2024sepvae}, and DoubleInfoGAN \cite{carton2024double} (our method is the first to use diffusion for CA). As discussed in Sec. \ref{sec:related_work}, these models rely on low-capacity latent bottlenecks. We use their default configurations.
\begin{table}[t]
\centering
\caption{Quantitative comparison on FFHQ (Glasses Attribute). We evaluate reconstruction fidelity and the precision of salient attribute swaps. Swapping metrics assess both the success of the transfer (Acc: glasses class accuracy) and the preservation of unrelated common attributes (Gender, Smile, Pose, ID-Sim \ref{app:cond_quant}) to measure latent leakage. Swapping results are macro-averaged over fine-grained classes (NoGlasses, ReadingGlasses, SunGlasses). Detailed  results can be found in Table \ref{tab:full_swapping_results}.
}
\label{tab:swap}
\resizebox{\textwidth}{!}{
\begin{tabular}{l | ccc || c|cccc}
\toprule
& \multicolumn{3}{c||}{\textbf{Reconstruction}} & \multicolumn{5}{c}{\textbf{Swapping Fine-grained classes (Glasses) }} \\
\cmidrule(lr){2-4} \cmidrule(lr){5-9}
\textbf{Method} & \textbf{SSIM} $\uparrow$ & \textbf{LPIPS} $\downarrow$ & \textbf{ID-Sim} $\uparrow$ & \textbf{Acc} $\uparrow$ & \textbf{ID-Sim}$\uparrow$ &\textbf{Gender} (acc) $\uparrow$ & \textbf{Smile} (acc) $\uparrow$ & \textbf{Pose} (MAE) $\downarrow$   \\
\midrule
MM-cVAE  & .370 & .643 & .307&11.8 &.302 &62.3 &57.7 & 4.62\\
SepVAE  & .358 & .642& .307& 19.6 &.301 &57.7 &56.4 & 8.29\\
D.InfoGAN & .326& .430 &.327 & 47.0 &.314 &68.5 &79.2 &  7.35\\
\midrule
\textbf{Diff-CA (Ours)} & \textbf{.610} & \textbf{.115} & \textbf{.522}& \textbf{94.5} &\textbf{.474} &\textbf{92.5} & \textbf{92.9}& \textbf{1.87}\\
\bottomrule
\end{tabular}
}
\end{table}

\paragraph{Reconstruction and swapping.} 
As shown in Fig.~\ref{fig:swapping_baselines}, VAE and GAN-based baselines produce blurred reconstructions and lose identity during swaps. By contrast, Diff-CA not only reconstructs faithfully from $\hat{Z}_S+\hat{Z}_C$, but also allows for the swapping of $\hat{Z_S}$ between images, transferring the target attribute without altering the rest of the image. 

\begin{figure}[t]
    \centering
\includegraphics[width=0.99\linewidth]{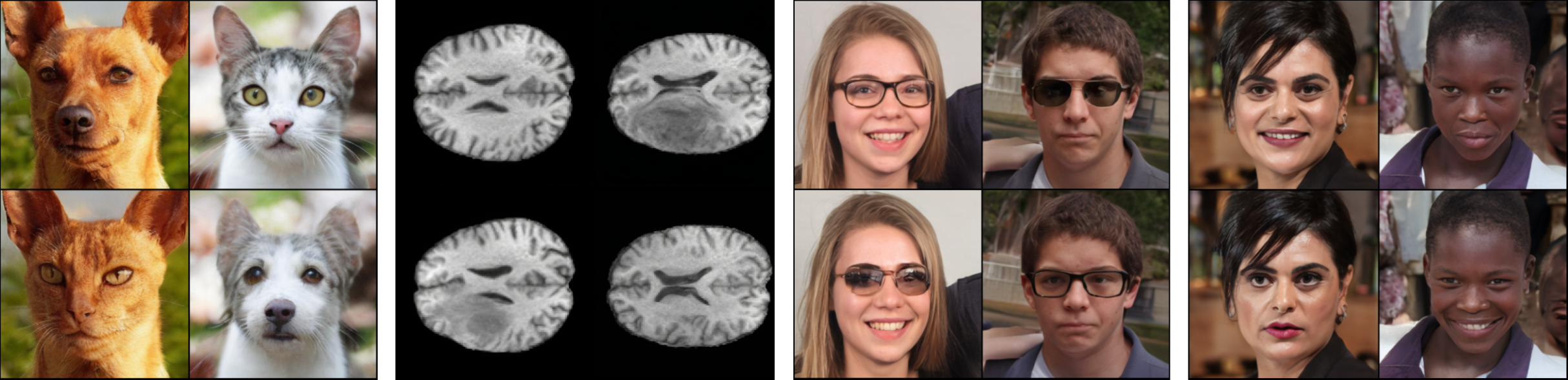}
    \caption{For each dataset, the first row represents two real images from background and target  while the second row shows swapping results with salient : ``cat'', ``tumor'', ``glasses'' and ``smile".}  
    \label{fig:main_swap}
\end{figure}


The superiority of our approach in Fig.~\ref{fig:swapping_baselines} over the others is further confirmed by the quantitative results in Tab.~\ref{tab:swap}, where Diff-CA substantially outperforms all baselines across every metric.

Beyond significantly better reconstruction performances, the most telling evidence of our decomposition quality lies in the swapping results. Diff-CA achieves 94.5\% fine-grained glasses classification accuracy after swap, more than doubling the best baseline (47.0\% for DoubleInfoGAN), demonstrating that the salient subspace captures not only coarse binary attributes but is furthermore structured to reflect fine-grained stylistic distinctions despite training under binary supervision only. 

\paragraph{Preservation of the common factors.} Crucially, this transfer is \emph{selective}: gender, smile, and head pose are preserved with 92.5\%, 92.9\%, and 1.87$^\circ$ MAE respectively, compared to 68.5\%, 79.2\%, and 7.35$^\circ$ for DoubleInfoGAN. The gap confirms that $\hat{Z}_S$ and $\hat{Z}_C$ are well disentangled, with minimal attribute leakage across subspaces.

Finally, Fig.~\ref{fig:main_swap} shows that this capacity for clean salient separation generalises beyond faces, extending to animal species (AFHQ) and pathological regions (BraTS), which underlines the domain-agnostic nature of our contrastive decomposition.

\subsection{Latent Space Analysis} 
\label{subsec:latent_analysis}

\paragraph{Salient interpolation.} To evaluate the smoothness and semantic coherence of the learned salient latent space, we perform linear interpolation between salient codes while fixing the common factor. Specifically, given two images $a$ and $b$, we extract their respective decompositions  $\hat{Z}^a_S, \hat{Z}^a_C$ and $\hat{Z}^b_S, \hat{Z}^b_C$, then generate images along the interpolation path $G(\hat{Z}^a_C + \alpha \hat{Z}^a_S + (1-\alpha) \hat{Z}^b_S, \varepsilon)$ for $\alpha \in [0, 1]$. As shown in Fig.~\ref{fig:salient-interpol}, this operation produces semantically meaningful intermediate results. For the glasses attribute (top row), the interpolation smoothly transitions between different styles of eyewear while preserving identity and other facial characteristics. For the cat-dog attribute (bottom row), interpolation reveals a progressive emergence of species-specific features (e.g., ear shape, whiskers) while maintaining pose and background. These results demonstrate that our learned $Z_S$ captures a continuous representation of the salient factor, enabling fine-grained control beyond the weak binary supervision used during training.

\begin{figure}[t]
    \centering
\includegraphics[width=0.99\linewidth]{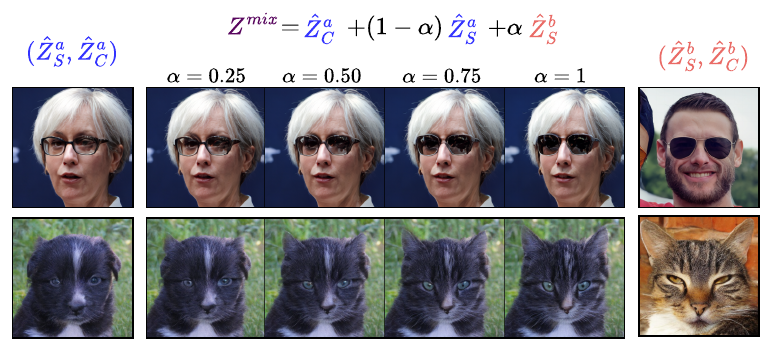}
    \caption{\textbf{Salient interpolation}: We fix the common $\hat{Z}_C$ of the left image and interpolate the salient $\hat{Z}_S$ with the salient of the right image. \textbf{Top}: The style of the glasses is progressively transferred. \textbf{Bottom}: cat features (ears, whiskers) progressively appear.}
    \label{fig:salient-interpol}
\end{figure}

\begin{figure}[t]
    \centering
    \begin{minipage}[t]{0.99\linewidth}
        \centering
        \includegraphics[width=0.37\linewidth]{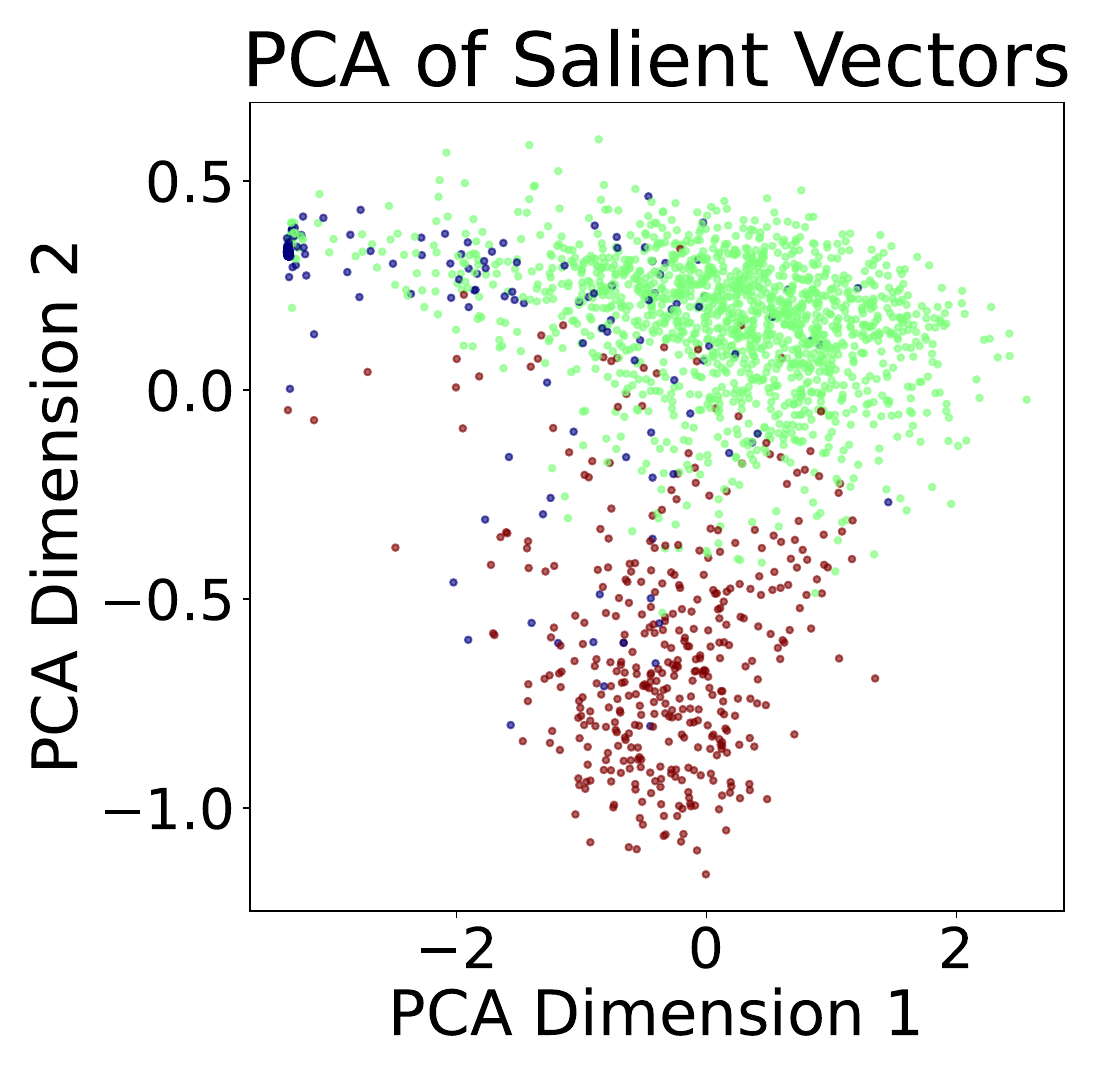}
        \centering
        \includegraphics[width=0.6\linewidth]{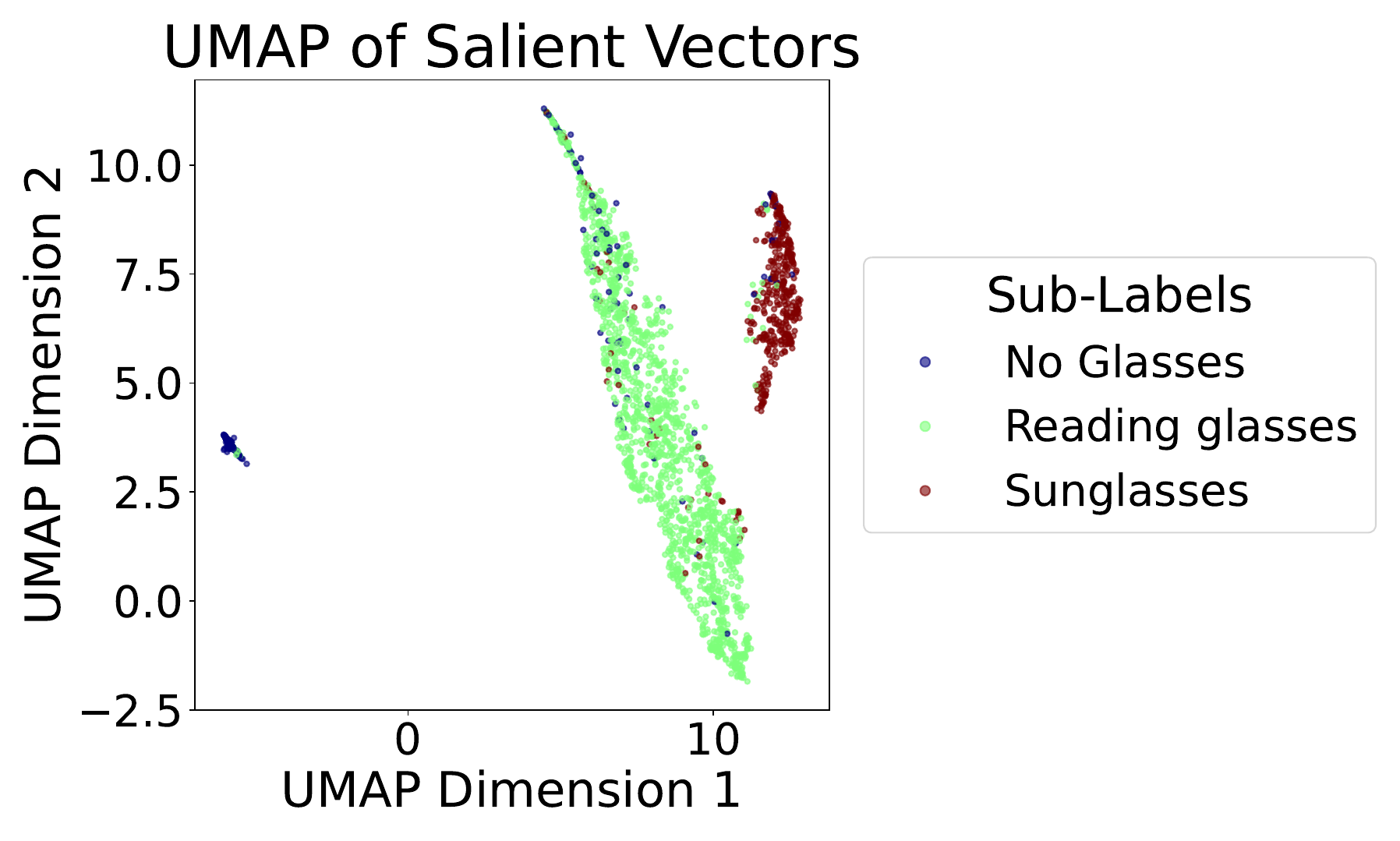}
    \caption{PCA and UMAP projection of $\hat{Z}_S$. Even though $\hat{Z}_S$ was only learned using weak binary supervision, fine-grained subclasses are still separated.}
    \label{fig:projection}
    \end{minipage}
\end{figure}
\paragraph{PCA and UMAP projections.} To assess whether the learned salient codes $\hat{Z}_S$ capture fine-grained structure beyond the coarse binary supervision, we analyze their geometric organization using dimensionality reduction. Figure~\ref{fig:projection} shows PCA and UMAP projections of salient vectors extracted from the weak supervision (Glasses vs. No Glasses). We plot the three fine-grained classes, No Glasses, Reading Glasses, and Sunglasses, in different colors. The three subclasses form clearly distinct clusters in the UMAP projection, indicating that $\hat{Z}_S$ spontaneously learns to separate fine-grained variations of the salient attribute without explicit supervision. 

\section{Conclusion and Future Work}
We propose a rigorous theoretical framework for Contrastive Analysis in the context of generative models. We proved that under an additive structural assumption with weak binary supervision, the true common and salient factors are uniquely identifiable (Th.~\ref{thm:identifiability}), and derived a lower bound on controllability based on mutual information that establishes the conditions for effective latent manipulation (Prop.~\ref{prop:control}).

Building on this theory, we introduced Diff-CA, a novel method that leverages both a high-fidelity editable conditioning space and a practical CA training framework. Our conditioning architecture constructs a latent space that satisfies the structural assumptions required by our identifiability results, while enabling precise semantic control. The encoder network, trained with our proposed losses learns to decompose conditioning tokens into common ($\hat{Z_C}$) and salient ($\hat{Z_S}$) factors that can be independently manipulated.

Diff-CA addresses a fundamental limitation of existing CA methods: the trade-off between reconstruction quality and editability. While prior CA approaches sacrifice image fidelity to achieve latent separation, our diffusion-based model maintains both. Empirically, Diff-CA achieves state-of-the-art reconstruction fidelity across multiple domains while simultaneously delivering state-of-the-art editing performance.

Our current implementation operates on 2D images and is not adapted to fully 3D volumetric data, such as high-resolution medical volumes, which we leave for future work. In addition, we focus on an additive factorization of common and salient components. Exploring alternative factorizations, as well as settings where salient factors appear in both distributions, is an important future research direction.

\section*{Acknowledgments}
We acknowledge the support of the French Agence Nationale de la Recherche (ANR) under reference ANR-21-CE23- 0024 IDEGEN, ANR-19-CE40-005 MISTIC and EUR BERTIP (ANR-18-EURE-0002). This work was performed using HPC resources from GENCI–IDRIS (A0160615058).

\bibliographystyle{plainnat}

{
\small
\bibliography{ref}
}
\clearpage
\appendix

\titleformat{\section}[block]    
  {\large\bfseries}        
  {\appendixname~\thesection} 
  {0.0em}                  
  {: }                     

\section{Broader Impact}
Our method has the potential to significantly impact domains such as medical imaging, where diffusion-based models combined with contrastive analysis could help identify subtle imaging patterns associated with specific pathologies that may not be visible to the naked eye of a clinician. This could lead to improved understanding of disease mechanisms and enable the discovery of novel, non-invasive imaging biomarkers. More broadly, structuring the representation space through contrastive objectives may enhance the trustworthiness and explainability of diffusion-based models for generation and editing by disentangling common and salient factors, thereby facilitating the identification of hidden biases or spurious shortcuts. However, these benefits must be balanced against potential risks: although Diff-CA primarily aims at enhancing explainability, improved generative and editing capabilities may still enable the creation of highly realistic synthetic content, raising concerns about deepfakes and misinformation. Furthermore, if not carefully designed, particularly in dataset construction, contrastive objectives may inadvertently reinforce existing biases or induce misleading feature separations, thereby undermining the intended gains in robustness and interpretability. Finally, AI decisions can affect an uncautious expert's final decision, which could have dramatic consequences in fields such as medicine. The use of contrastive analysis (CA) and, more broadly, artificial intelligence should be conducted with all necessary critical thinking and should NOT replace, but rather support, the expert's work or analysis.

\section{Additional Results and Discussion on T2I Models}
\label{sec:LLMs_failure}

\begin{figure}[!ht]    \centering\includegraphics[width=0.99\linewidth]{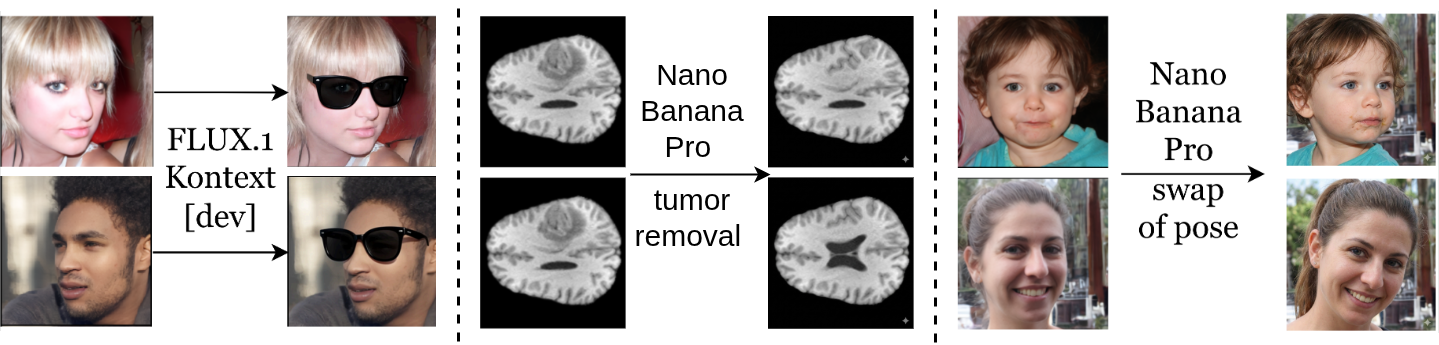}
    \caption{\textbf{Failure cases} of some prompt-based diffusion editing methods. \textbf{Left}: FLUX.1 Kontext adds the same generic sunglasses to everyone. \textbf{Middle}: Nano Banana Pro correctly removes the tumor but fails to generate an anatomically realistic image (the generated brain gyrification is not realistic), and it alters the healthy anatomy (ventricles in the bottom image). \textbf{Right}: Nano Banana Pro fails to swap head position without altering other attributes (\textit{e.g.,} background, hair, colors).}
    \label{fig:prompt_failure}
\end{figure}

\begin{figure}[t]    \centering\includegraphics[width=0.99\linewidth]{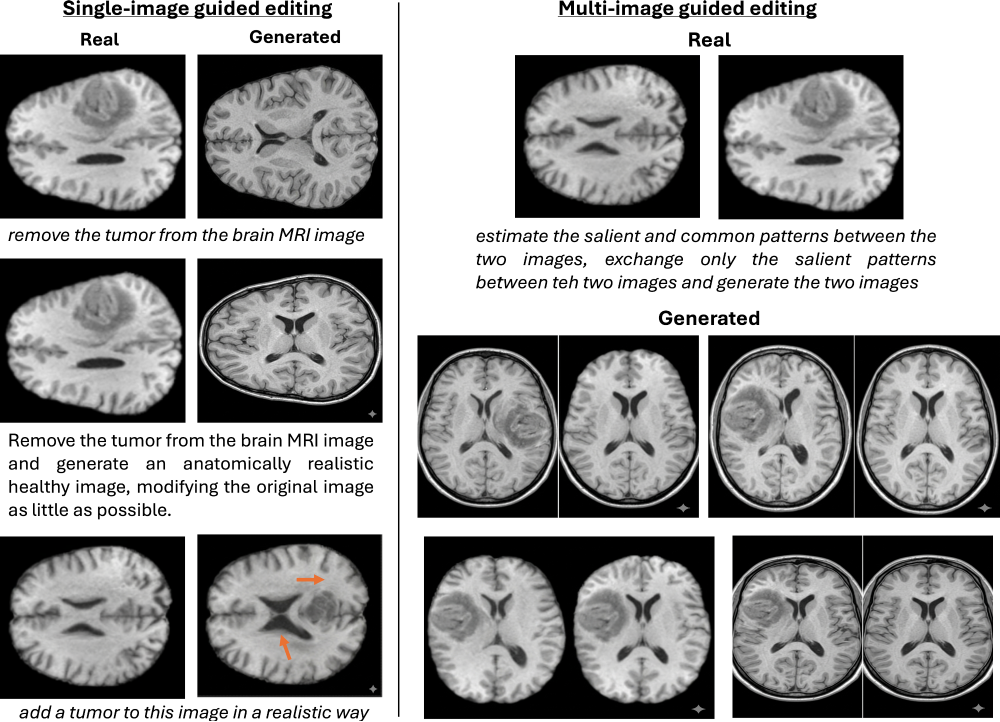}
    \caption{\textbf{Failure cases} of Nano Banana 2 (Gemini 3 Flash Image). \textbf{Left}: We use a single real image and a prompt to guide the editing. All generated images are either anatomically unrealistic or show altered anatomy (orange arrows). \textbf{Right}: Two real images (without and with a tumor) are provided as input, along with a prompt to guide the editing. This setup more closely resembles the goal of Contrastive Analysis (CA), although CA is a population-based statistical method rather than an image editing approach based on two input images. The different pairs of generated images correspond to outputs generated sequentially (one after the other) using the same prompt, with the addition of the instruction ``the previous images are wrong''. This shows that the results can vary depending on the prompt (low robustness) and that, even with corrective prompts, the model fails to produce a convincing solution. Indeed, all generated images exhibit differences in orientation and anatomical structures compared to the original images (\textit{e.g.,} skull, gyrification, ventricles, shape), and in some cases the tumor (i.e., the salient feature) is not correctly added or removed.}
\label{fig:prompt_failure_banana2}
\end{figure}

\addcontentsline{toc}{section}{\appendixname~\thesection: Additional Results on T2I Models}
Text-to-image (T2I) methods have seen massive adoption for general-purpose image editing, yet they exhibit critical structural limitations when precise, controlled, or out-of-domain operations are required. We discuss four such limitations that are particularly relevant to the CA setting.

\paragraph{Output randomization.} Across repeated runs with identical inputs and prompts, T2I models produce meaningfully different outputs. On medical images, for instance, repeated queries yield varying degrees of lesion removal, inconsistent tissue intensities, and unstable anatomical geometry. In industrial and medical contexts, reproducibility is a hard requirement for validation and regulatory compliance. A method that produces different results across runs cannot be audited or certified regardless of its average visual quality.

\paragraph{Distributional implausibility.} Even visually plausible outputs can be clinically or scientifically incorrect, a failure mode that is silent and therefore particularly dangerous. In our MRI experiments, prompted baselines and our own replications produce  brain ventricle and gyrification geometries inconsistent with healthy anatomy, yet these outputs appear reasonable to a non-expert. In any domain requiring distributional conformity, outputs that fall outside the target distribution are more problematic than obvious failures. By performing exact latent swapping within a model trained on the target distribution, our approach preserves individual anatomy by construction.

\paragraph{Prompt sensitivity.} The quality of T2I outputs is highly sensitive to the precise wording of the prompt, a choice entirely decoupled from the actual content of interest. This uncontrolled variable makes deployment impractical in settings where the salient concept is not easily verbalized, or where prompt standardization across users and institutions cannot be guaranteed. CA requires no such verbalization, which is precisely its motivation.

\paragraph{Domain-agnostic safety filters.} General-purpose generators apply content filters that can interfere unpredictably with scientific use cases. In our experiments, pathology re-introduction (a natural reversibility check) was refused by the model in several runs. Models trained specifically for a target domain and task, as in our framework, operate without such constraints.

\begin{figure}[t]
    \centering
    \includegraphics[width=0.96\linewidth]{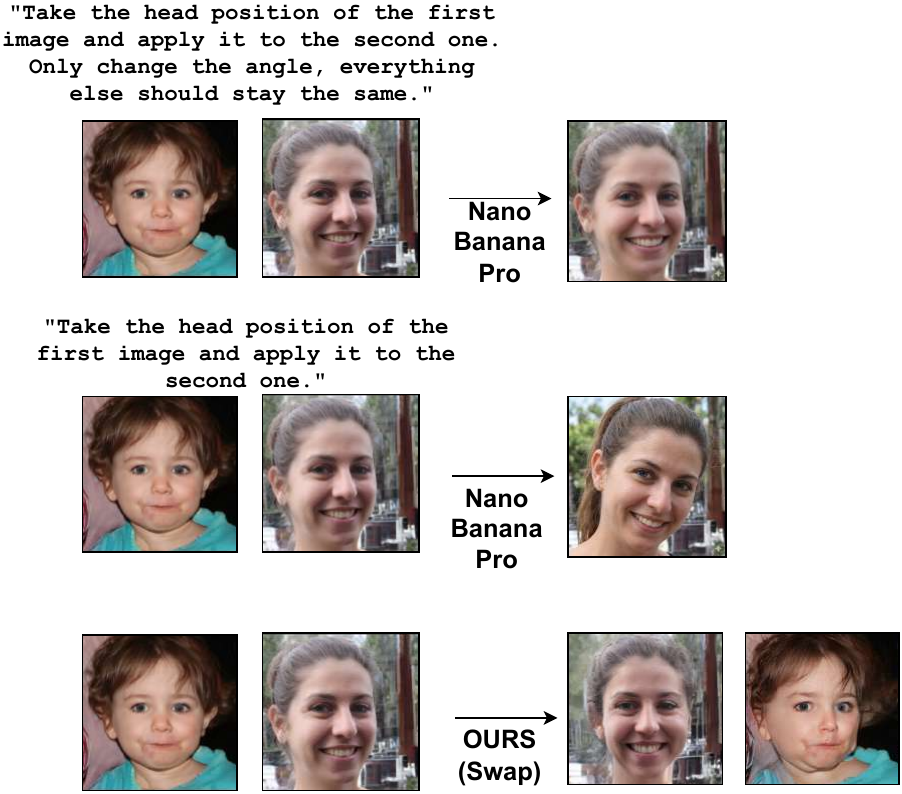}
    \caption{Failure cases of Nano Banana Pro for a head position transfer for two different prompts, and our swapping model.}
    \label{fig:gemini_fail}
\end{figure}

The most widely used commercial models include Nano Banana Pro/2 \cite{google2025nanobanana} and ChatGPT, which rely on in-context information and can technically use multiple images alongside a text prompt to perform edits and swapping. In Fig.~\ref{fig:prompt_failure}, Fig.~\ref{fig:prompt_failure_banana2} and Fig.~\ref{fig:gemini_fail}, we illustrate failure cases for both single-image and multi-image guided editing,  
highlighting the importance of specialized models for critical tasks, such as medical imaging.

\section{Theory}
\label{app:TheoryAndProofs}

\subsection{Proof of Theorem \ref{thm:identifiability}}
\label{app:ident_proof}
We prove the following result:
\addtocounter{theorem}{-1}
\begin{theorem}[Identifiability of Additive Factors]
Assume the conditioning $Z$ is generated by the direct sum of independent factors $Z_S$ and $Z_C$, where $Z_S$ vanishes for background samples ($Y=0$). If a learned decomposition $ Z \mapsto (\hat{Z}_S, \hat{Z}_C)$ satisfies:
\begin{enumerate}[leftmargin=1cm, topsep=0pt]
    \item \textbf{Additive Reconstruction:} The learned codes sum to the input: $\hat{Z}_S + \hat{Z}_C = Z$ 
    \item \textbf{Pinning:} The salient code vanishes only for background samples: $Y=0 \Longleftrightarrow \hat{Z}_S = 0$.
    \item \textbf{Independence:} The learned codes are independent: $\hat{Z}_S \indep \hat{Z}_C$.

\end{enumerate}
Then, the learned codes uniquely identify the true encoded factors:
\begin{equation}
    \hat{Z}_S = Z_S \quad \text{and} \quad \hat{Z}_C = Z_C
\end{equation}
\end{theorem}

To provide a rigorous proof, we restate the conditions from the main text as formal hypotheses:
\begin{enumerate}[label=\textbf{H\arabic*}]
    \item \textbf{(Latent Direct Sum):} $Z = Z_S + Z_C$, where $Z_S \in \mathcal{S}$, $Z_C \in \mathcal{C}$, and $\mathbb{R}^d = \mathcal{S} \oplus \mathcal{C}$.
    \item \textbf{(Latent Independence):} $Z_S \indep Z_C$ and $\mathbb{P}(Z_S=0) > 0$.
    \item \textbf{(Code Decomposition):} The latent codes additively reconstruct $Z$: $Z = \hat{Z}_S + \hat{Z}_C$.
    \item \textbf{(Code Independence):} The learned factors are independent: $\hat{Z}_S \indep \hat{Z}_C$.
    \item \textbf{(Pinning):} The absence of the attribute in the signal implies a zero code: $\hat{Z}_S = 0 \iff Z_S = 0$.
\end{enumerate}

\begin{proof}
The proof is in two steps. In the first step, we identify the marginal laws of $\hat{Z}_S$ and $\hat{Z}_C$ as the laws of $Z_S$ and $Z_C$, respectively. In the second step, we extend the equality in law to equality \textit{almost surely}, by showing that the leakage variable $L \coloneqq \hat{Z}_S - Z_S$ is almost surely zero. Note that by the additive structure (\textbf{H3}), we have $\hat{Z}_C - Z_C = -L$ a.s.

\medskip\noindent\textbf{Step 1: Identification of marginals.}

We first establish that the learned factors match the true marginal distributions.

Consider the event $\{Y=0\}$. By the pinning assumption \textbf{H5}, $\hat{Z}_S=0$. Consequently, conditioned on $Y=0$, we have $Z = \hat{Z}_C$.

By the data generation assumption, $Y=0 \implies Z_S=0$, so, conditioned on $Y=0$, we have $Z = Z_C$.

Thus, conditioned on $Y=0$, we have $\hat{Z}_C = Z_C$. We now generalize to the unconditional result.
\begin{itemize}
    \item By \textbf{H4}, $\hat{Z}_C \indep \hat{Z}_S$. Since $Y = \mathds{1}\{\hat{Z}_S \neq 0\}$ is a function of $\hat{Z}_S$, we have $\hat{Z}_C \indep Y$.
    \item By \textbf{H2}, $Z_C \indep Z_S$. Since $Y = \mathds{1}\{S \neq 0\}$ is function of $S$, we have  $Z_C \indep Y$.
\end{itemize}
Since $\hat{Z}_C = Z_C$ on $\{Y=0\}$, the two previous points allow us to extend the equality in law to all values of $Y$:
\begin{equation}
    \hat{Z}_C \overset{Law}{=} Z_C
\end{equation}

We know show equality of the marginal laws of $Z_S$ and $\hat{Z}_S$.

We denote by $\phi_X(t) = \mathbb{E}[e^{i\langle X; t\rangle}]$ the characteristic function of a variable $X$. Since $Z = \hat{Z}_S + \hat{Z}_C = Z_S + Z_C$ involves sums of independent variables, the characteristic functions satisfy $\phi_{\hat{Z}_S}\phi_{\hat{Z}_C} = \phi_{Z_S}\phi_{Z_C}$.

A characteristic function is continuous on $\mathbb{R}$ and takes the value 1 at 0. Therefore, all characteristic functions are non-zero around 0. Thus, on a neighborhood $U$ around 0, we have:
\begin{equation}
    \forall t \in U, \phi_{\hat{Z}_S}(t) = \phi_{Z_S}(t)
\end{equation}
Assuming the encoded attributes possess a moment-generating function (satisfied by the bounded DINO-token space), the characteristic functions are analytic, and the identity extends to $\mathbb{R}^d$ by the identity theorem for analytic functions. Thus, we have:
\begin{equation}
    \hat{Z}_S \overset{Law}{=} Z_S
\end{equation}

\medskip\noindent\textbf{Step 2: Vanishing leakage via gradients.}
We leverage the independence of the learned codes to constrain the leakage variable $L= \hat{Z}_S - Z_S$. We will first show that it's zero \textit{a.s.} conditioned on $Z$, then show it's zero \textit{a.s.} unconditionnaly. 

By independence \textbf{H4}, for all ${u, v \in \mathbb{R}^d}$:
\begin{equation}
    \mathbb{E}\left[ e^{i\langle u, \hat{Z}_S \rangle + i\langle v, \hat{Z}_C \rangle} \right] = \phi_{\hat{Z}_S}(u)\phi_{\hat{Z}_C}(v)
\end{equation}
Substituting the relations $\hat{Z}_S = Z_S + L$ and $\hat{Z}_C = Z_C - L$ on the LHS, and using the marginal identification from Step 1 to replace the RHS characteristic functions with those of the true factors, we get:
\begin{equation}
    \label{eq:char_expansion}
    \mathbb{E}\left[ e^{i\langle u, Z_S \rangle + i\langle v, Z_C \rangle} e^{i\langle u-v, L \rangle} \right] = \phi_{Z_S}(u)\phi_{Z_C}(v)
\end{equation}
We differentiate both sides with respect to $u$. Assuming finite first moments, we apply the gradient $\nabla_u$ under the expectation:
\begin{align*}
    \nabla_u \text{LHS} &= \nabla_u\mathbb{E}\left[ e^{i\langle u, Z_S \rangle + i\langle v, Z_C \rangle} e^{i\langle u-v, L \rangle} \right] \\
    &=\mathbb{E}\left[\nabla_u e^{i\langle u, Z_S \rangle + i\langle v, Z_C \rangle} e^{i\langle u-v, L \rangle} \right]\\
    &=\mathbb{E}\left[ (iZ_S + iL) \, e^{i\langle u, Z_S \rangle + i\langle v, Z_C \rangle} e^{i\langle u-v, L \rangle} \right] \\
    &=i \mathbb{E}\left[ (Z_S + L) \, e^{i\langle u, Z_S \rangle + i\langle v, Z_C \rangle} e^{i\langle u-v, L \rangle} \right]\\
    \nabla_u \text{RHS} &= \nabla_u \left[\phi_{Z_S}(u) \, \phi_{Z_C}(v)\right] \\
    &=  \nabla_u \left[\phi_{Z_S}(u)\right]\phi_{Z_C}(v)\\
    &=  \mathbb{E}[iZ_Se^{i\langle u, Z_S \rangle}] \phi_{Z_C}(v)\\
    &=  i\mathbb{E}[iZ_Se^{\langle u, Z_S \rangle}] \phi_{Z_C}(v)
\end{align*}
Dividing by $i$ on both sides we get:
\begin{align}
    &\ \ \mathbb{E}\left[ (Z_S + L) \, e^{i\langle u, Z_S \rangle + i\langle v, Z_C \rangle} e^{i\langle u-v, L \rangle} \right] \\
    = &\ \ \mathbb{E}[Z_Se^{\langle u, Z_S \rangle}] \phi_{Z_C}(v)
\end{align}
We evaluate this equation on the hyperplane $u = v = t$. The leakage term in the exponent vanishes ($e^{i\langle 0, L \rangle} = 1$), yielding:
\begin{equation}
    \label{eq:gradient_eval}
    \mathbb{E}\left[ (Z_S + L) \, e^{i\langle t, Z_S + Z_C \rangle} \right] = \mathbb{E}[Z_Se^{i\langle t, Z_S \rangle}] \phi_{Z_C}(t)
\end{equation}
Using the independence $Z_S \indep Z_C$, the RHS can be rewritten as $\mathbb{E}[Z_Se^{i\langle t, Z_S \rangle} e^{i\langle t, Z_C \rangle}]$.
Subtracting the RHS from the LHS in \eqref{eq:gradient_eval}, we obtain:
\begin{equation}
    \mathbb{E}\left[ (Z_S + L) \, e^{i\langle t, Z_S + Z_C \rangle}  - Z_Se^{i\langle t, Z_S \rangle} e^{i\langle t, Z_C \rangle}\right] =0\
\end{equation}
The terms $Z_Se^{i\langle t, Z_S \rangle} e^{i\langle t, Z_C \rangle}$ cancel out, leaving:
\begin{align}
        \forall t \in \mathbb{R}^d, \quad \mathbb{E}\left[ L \, e^{i\langle t, Z_S+Z_C \rangle} \right] = 0\\
        \mathbb{E}\left[ L \, e^{i\langle t, Z \rangle} \right] = 0\\
         \mathbb{E}\left[ \mathbb{E} \left[L \, e^{i\langle t, Z \rangle}|Z \right]\right] = 0\\
         \int_{\mathbb R^d} \mathbb{E} \left[L \, e^{i\langle t, z \rangle}|Z=z \right] d\mathbb P_Z(z)&=0
\end{align}
In the integral above, $d\mathbb{P}_Z(z)$ is the probability measure of $Z$. We define a new signed measure $\mu$ such that ${d\mu(z) = \mathbb{E}[L \mid Z=z] d\mathbb{P}_Z(z)}$.

In this case, 
\begin{equation}
    \int e^{i\langle t, z \rangle} d\mu(z) = 0
\end{equation}
is the Fourier-Stieltjes transform of the measure $\mu$.
By the uniqueness of the Fourier transform, if the transform is zero for all $t$, then the measure $\mu$ must be zero everywhere. 

For $\mu$ to be the zero measure, its density with respect to $\mathbb{P}_Z$ must vanish for $\mathbb{P}_Z$-almost every $z$:
\begin{equation}
    \mathbb{E}[L \mid Z] = 0 \quad  \mathbb{P}_Z\text{--almost surely.}
\end{equation}
Because $\hat{Z}_S$ is produced by a deterministic encoder $E_{\theta_E}(Z)$ and $Z_S$ is a deterministic projection of $Z$ , the difference $L = \hat{Z}_S - Z_S$ is perfectly known once $Z$ is known (i.e., it is $\sigma(Z)$-measurable).

Therefore, $\mathbb{E}[L \mid Z] = L$ a.s.

Combined with the Fourier result $\mathbb{E}[L \mid Z] = 0$, we conclude $L = 0$ almost surely. Therefore, we have the equality \textit{a.s.} between the true factors and the encoded factors.
\end{proof}

\subsection{Proof of Proposition \ref{prop:control}}\label{app:control_proof}
Assume that the generator $G_{\psi}$ is deterministic given $(Z, \varepsilon)$, and let $\hat{X}=G_{\psi}(Z,\varepsilon)$ with corresponding attributes $\hat{S}$ and $\hat{C}$. Then for any $(\hat{Z}_S, \hat{Z}_C)$ that are functions of $Z$, the control over the salient attribute is lower-bounded by:
\begin{equation}
I(\hat{S}; \hat{Z}_S) \ge \underbrace{I(S; \hat{S})}_{\text{Fidelity}} - \underbrace{I(\hat{S}; \varepsilon | Z)}_{\text{Noise dependence}} - \underbrace{I(S; \hat{Z}_C | \hat{Z}_S)}_{\text{Entanglement}}
\label{eq:control_decomposition}
\end{equation}
Similarly, for the common context:
\begin{equation}
I(\hat{C}; \hat{Z}_C) \ge I(C; \hat{C}) - I(\hat{C}; \varepsilon | Z) - I(C; \hat{Z}_S | \hat{Z}_C)
\end{equation}
\begin{proof}
We prove the result for $S$ and $Z_C$. The proof for $C$ and $Z_C$ is analogous.\\

\textbf{Step 1: Decomposition.} 
We start with the identity $I(\hat{S}; \hat{Z}_S) = H(\hat{S}) - H(\hat{S}|\hat{Z}_S)$.
We first decompose the conditional entropy term. Using the definition of conditional mutual information, we write:
\begin{equation}
    H(\hat{S}|\hat{Z}_S) = I(\hat{S}; \hat{Z}_C | \hat{Z}_S) + H(\hat{S} | \hat{Z}_S, \hat{Z}_C).
\end{equation}
Since the pair $(\hat{Z}_S, \hat{Z}_C)$ uniquely determines the conditioning $Z$, we have $H(\hat{S} | \hat{Z}_S, \hat{Z}_C) = H(\hat{S}|Z)$. Furthermore, because $\hat{S}$ is determined by $Z$ and $\varepsilon$, any remaining entropy in $\hat{S}$ given $Z$ is attributable entirely to the noise $\varepsilon$:
\begin{equation}
    H(\hat{S}|Z) = I(\hat{S}; \varepsilon | Z) + H(\hat{S}|Z, \varepsilon) = I(\hat{S}; \varepsilon | Z),
\end{equation}
where $H(\hat{S}|Z, \varepsilon) = 0$ due to the determinism assumption. Substituting these back into the entropy identity yields:
\begin{equation}
    I(\hat{S}; \hat{Z}_S) = H(\hat{S}) - I(\hat{S}; \hat{Z}_C | \hat{Z}_S) - I(\hat{S}; \varepsilon | Z).
    \label{eq:decomp_step1}
\end{equation}

\textbf{Step 2: Bounding the Entanglement.}
We now bound the entanglement term $I(\hat{S}; \hat{Z}_C | \hat{Z}_S)$ to introduce the true attribute $S$. By the chain rule of mutual information, adding a variable cannot decrease the information:
\begin{equation}
    I(\hat{S}; \hat{Z}_C | \hat{Z}_S) \le I(\hat{S}, S; \hat{Z}_C | \hat{Z}_S).
\end{equation}
Expanding the right-hand side via the chain rule:
\begin{equation}
    I(\hat{S}, S; \hat{Z}_C | \hat{Z}_S) = I(S; \hat{Z}_C | \hat{Z}_S) + I(\hat{S}; \hat{Z}_C | \hat{Z}_S, S).
\end{equation}
We bound the second term using the property that mutual information is bounded by discrete entropy, $I(X;Y|Z) \le H(X|Z)$:
\begin{equation}
    I(\hat{S}; \hat{Z}_C | \hat{Z}_S, S) \le H(\hat{S} | \hat{Z}_S, S) \le H(\hat{S}|S).
\end{equation}
Combining these inequalities gives:
\begin{equation}
    I(\hat{S}; \hat{Z}_C | \hat{Z}_S) \le I(S; \hat{Z}_C | \hat{Z}_S) + H(\hat{S}|S).
    \label{eq:bound_step2}
\end{equation}

\textbf{Step 3: Rearrangement.}
Substituting the bound from \eqref{eq:bound_step2} into \eqref{eq:decomp_step1}:
\begin{equation}
    I(\hat{S}; \hat{Z}_S) \ge H(\hat{S}) - \left[ I(S; \hat{Z}_C | \hat{Z}_S) + H(\hat{S}|S) \right] - I(\hat{S}; \varepsilon | Z).
\end{equation}
We group the entropy terms. Recalling that $I(S; \hat{S}) = H(\hat{S}) - H(\hat{S}|S)$, we obtain the final lower bound:
\begin{equation}
    I(\hat{S}; \hat{Z}_S) \ge \underbrace{I(S; \hat{S})}_{\text{Fidelity}} - \underbrace{I(\hat{S}; \varepsilon | Z)}_{\text{Noise Dependence}} - \underbrace{I(S; \hat{Z}_C | \hat{Z}_S)}_{\text{Entanglement}}.
\end{equation}



Additionnaly, by the data-processing theorem, since $\hat{S}$ is a function of $\hat{X}$, we have $I(\hat{S}; \varepsilon | Z) \le I(\hat{X}; \varepsilon | Z)$.

Therefore:
\begin{equation}
    I(\hat{S}; \hat{Z}_S) \geq \underbrace{I(S; \hat{S})}_{\text{Fidelity}} - \underbrace{I(\hat{X}; \varepsilon | Z)}_{\text{Noise Dependence}} - \underbrace{I(S; \hat{Z}_C | \hat{Z}_S)}_{\text{Entanglement}}
\end{equation}
\end{proof}

\subsection{Link Between Cycle consistency and Independence}
\label{app:cycle}
In Sec.~\ref{sec:theory}, we introduce a cycle consistency loss $\mathcal{L}_{cycle}$ and claim that it helps enforce independence between latent codes $\hat{Z}_S$ and $\hat{Z}_C$. More precisely, we claim the following:

\begin{proposition}
Let $\hat{Z}_C^a$ and $\hat{Z}_S^b$ be latent codes extracted from independently sampled images $a$ and $b$ with independent generative factors $C$ and $S$. Define the mixed latent code
\[
Z^{\mathrm{mix}} = \hat{Z}_C^a + \hat{Z}_S^b,
\]
and let 
\[
(\hat{Z}_S^{\mathrm{mix}}, \hat{Z}_C^{\mathrm{mix}}) = E_{\theta_E}(Z^{\mathrm{mix}})
\]
be the separator outputs. Then, minimizing the cycle-consistency loss
\[
\mathcal{L}_{\mathrm{cyc}} = 
\mathbb{E}\Big[\|\hat{Z}_C^{\mathrm{mix}} - \hat{Z}_C^a\|_2^2 + \|\hat{Z}_S^{\mathrm{mix}} - \hat{Z}_S^b\|_2^2 \Big]
\]
enforces statistical independence between $\hat{Z}_S$ and $\hat{Z}_C$.
\end{proposition}

\begin{proof}
Given $Z^{\mathrm{mix}} = \hat{Z}_C^a + \hat{Z}_S^b$, since $a$ and $b$ are independent, it follows that also $\hat{Z}_C^a$ and $\hat{Z}_S^b$ are independent: $\hat{Z}_C^a \indep \hat{Z}_S^b$. If the cycle consistency loss is minimized, namely $\hat{Z}_S^{mix} = \hat{Z}_S^b$ and $\hat{Z}_C^{mix} = \hat{Z}_C^a$, then we also obtain $\hat{Z}_C^{mix} \indep \hat{Z}_S^{mix}$. Finally, assuming that mixed latents follow the same law as real samples ($\hat{Z}^{mix} \overset{Law}{=} Z$), we have $E_{\theta_E}(Z_{mix}) \overset{Law}{=}  E_{\theta_E}(Z)$. Since $E_{\theta_E}(Z_{mix}) = (\hat{Z}_S^{mix},\hat{Z}_C^{mix}) = (\hat{Z}_S^b, \hat{Z}_C^a)$ and $E_{\theta_E}(Z)=(\hat{Z}_S,\hat{Z}_C)$, we obtain $(\hat{Z}_S, \hat{Z}_C) \overset{Law}{=} (\hat{Z}_S^a, \hat{Z}_C^b)$, which entails that  $\hat{Z}_S \indep \hat{Z}_C$.
\end{proof}

The assumption $\hat{Z}^{mix} \overset{Law}{=} Z$ is not enforced with a loss, but is rather a consequence of our training paradigm for cycle consistency, as explained in App.~\ref{app:separation_details}.

\section{Implementation Details for Fine-Grained Control}
\label{app:implementation}

 \subsection{A Diffusion-based Conditioning Architecture}
\label{sec:finegrain}
\paragraph{Conditioning as the Primary Signal.}
To strictly control salient and common factors via the conditioning $Z$, we must ensure that $Z$ is the sole driver of semantic content, minimizing information leakage into the noise $\varepsilon$ used for generation (see Figure~\ref{fig:factors}). 
Standard prompt-based methods inject features into frozen text-to-image backbones, often failing to override the base model's priors and leading to identity loss or style drift (Fig.~\ref{fig:prompt_failure}). Consequently, as stated in Sec.~\ref{subsec:conditioning_pipeline}, we train a generative model \emph{from scratch} using image-only conditioning, derived from DINOv3 \cite{simeoni2025dinov3} embeddings, consisting in $K$ tokens. We adopt a Flow Matching objective, which enforces a deterministic mapping between noise and data, ensuring that visual fidelity is entirely determined by the structure of $Z$ rather than stochastic sampling paths.

\subsubsection{Architecture for Low Noise Dependence}
\label{app:subsec:conditioning_pipeline}

The generator $G_\psi$ is built conditioned on a structured token sequence $Z$ (Fig.~\ref{fig:combined_pipeline_and_query}). It operates on VAE latents and is trained using Optimal Transport Conditional Flow Matching (OT-CFM) \cite{lipman2022flow, tong2024improving}, which encourages straight trajectories in the latent space.

\paragraph{Latent Flow Matching.}Let $u_1 \in \mathbb{R}^d$ denote the image latents encoded by a frozen VAE. We define a linear probability path interpolating between a source distribution $p_0$ (standard Gaussian noise) and the target data distribution $p_1$ (image latents):
\begin{equation}
     u_t = (1-t)u_0 + t u_1, \quad t \in [0,1]
\end{equation}
where $u_0 \sim \mathcal{N}(0, I)$ and $u_1$ is the data latent. We train a U-Net vector field estimator $v_\psi$ to predict the velocity of this path, minimizing the flow matching objective:
\begin{equation}
\label{eq:flow_matching_loss}
     \mathcal{L}_{\text{FM}} = \mathbb{E}_{t, u_0, u_1} \Bigl[ \bigl\| v_\psi(u_t, t, Z) - (u_1 - u_0) \bigr\|^2 \Bigr].
\end{equation}
This formulation yields a deterministic ODE solver at inference time, enabling higher fidelity reconstruction and more stable latent manipulations than stochastic sampling.

\paragraph{Structured Token Conditioning.}
The conditioning signal $Z \in \mathbb{R}^{K\times d_{cond}}$ must capture both high-level semantics and low-level appearance. We construct $Z$ by concatenating two complementary token streams:

\begin{itemize}[leftmargin=*, topsep=0pt]
\item \textbf{Semantic Tokens $T$}: We extract features $F_{\text{dino}}$ from a frozen DINOv3 encoder. A lightweight cross-query module with $K-1$ learnable queries base on Perceiver-IO \cite{jaegle2021perceiver} compresses these features into a fixed-length sequence $T \in \mathbb{R}^{(K-1) \times d_{\text{cond}}}$. This distills spatial semantics (e.g., ``eyes'', ``texture'') into a compact representation. This pipeline is illustrated in Fig. \ref{fig:cross_query} with further details in Tab. \ref{tab:cross_query}.
\item \textbf{Color Token} $t_{\text{col}}$: Diffusion models have trouble recovring low-frequency information \cite{zhang2025antiexposure, shum2025color}. Additionally, the distilled DINO features can lose color information. We introduce a single token $t_{\text{col}}$ derived from a shallow CNN encoder trained to recover histogram information and mitigate color-shifting. The specification are described in Tab. \ref{tab:color_encoder}.
\end{itemize}
\begin{figure}[h]
    \centering
    \includegraphics[width=0.98\linewidth]{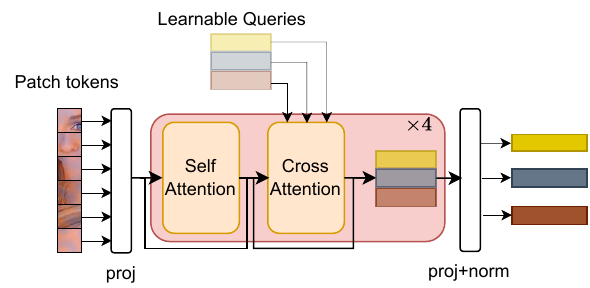}
    \caption{Diagram of our Cross-query extractor module}
    \label{fig:cross_query}
\end{figure}
The final conditioning $Z := [T; t_{\text{col}}]$ is injected via cross-attention. This design forces the flow model to rely on $T$ for structure and $t_{\text{col}}$ for colorimetric information. 
\begin{figure}[t]
    \centering
\includegraphics[width=0.99\linewidth]{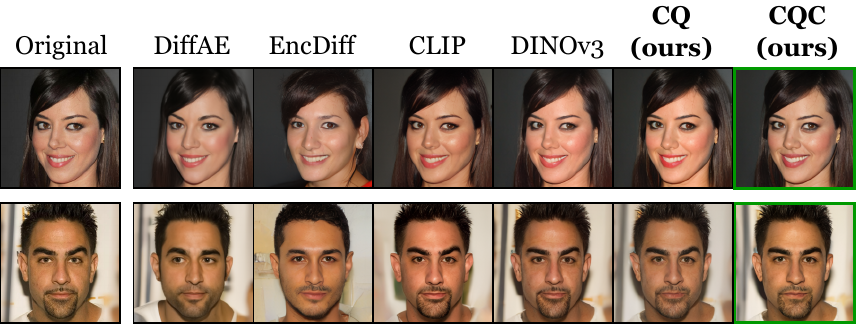}
    \caption{Reconstruction from random noise for different conditioning encoders. Our conditioning preserves identity, local details, and color. Zoom-in for details. Additional results in App. \ref{app:cond_qualitative}.}
    \label{fig:recon_grid}
\end{figure}


\subsubsection{Empirical Validation of Structural Assumptions}

We evaluate our conditioning architecture against several established baselines: DiffAE~\cite{preechakul2022diffusion}, which utilizes a ResNet to produce a single 512-d conditioning vector; EncDiff~\cite{yang2024diffusion}, which trains a CNN to generate $32 \times 128$ tokens; and projected variants of CLIP and DINOv3. For the latter, we utilize all last-layer activations (197 and 261 tokens respectively) and project each to 128-$d$ before conditioning. All models are trained on FFHQ~\cite{karras2019style} and tested on CelebA-HQ~\cite{liu2015faceattributes}.

\noindent \textbf{High-Fidelity Reconstruction.} As summarized in Tab.~\ref{tab:ffhq_comparison}, our Cross-Query and Color (CQC) conditioning achieves a very good LPIPS and FID, as well as a correct SSIM: this indicates that the model successfully reconstructs the major semantic attributes of our image, and relies on the noise to retrieve high-frequency textural details (e.g. \textit{hair}).  Visual results in Figure~\ref{fig:recon_grid} confirm that our model accurately preserves identity and fine structural details when sampling from random noise. Importantly, the inclusion of the color token $t_{col}$ mitigates the histogram mismatch that can be observed in vanilla Cross-Query reconstructions. See App.~\ref{app:cond_quant} for exhaustive reconstruction metrics across multiple datasets and App.~\ref{app:cond_qualitative} for further qualitative analysis.

\noindent \textbf{Latent Manipulation.} To assess the geometric and semantic properties of the conditioning space $Z$, we perform linear operations within the token manifold. 
\begin{itemize}[leftmargin=*,topsep=2pt,itemsep=1pt]
    \item \textbf{Interpolation}: Fig.~\ref{fig:recon_and_interp} (Right) shows that linear interpolation between $Z$ codes of two identities yields smooth, artifact-free transitions, indicating a continuous and well-structured representation, compared to DINOv3 which displays a "cross-fading" effect.
    \item \textbf{PCA-based Editing}: We further probe the semantic linearity of the token space by exploring its principal components. As illustrated in Fig.~\ref{fig:interpolation_pca}, moving linearly along these directions induces semantically valid changes in the output images, supporting the linear  assumption of our conditioning space.
\end{itemize}

\begin{figure}
    \centering
    \includegraphics[width=1\linewidth]{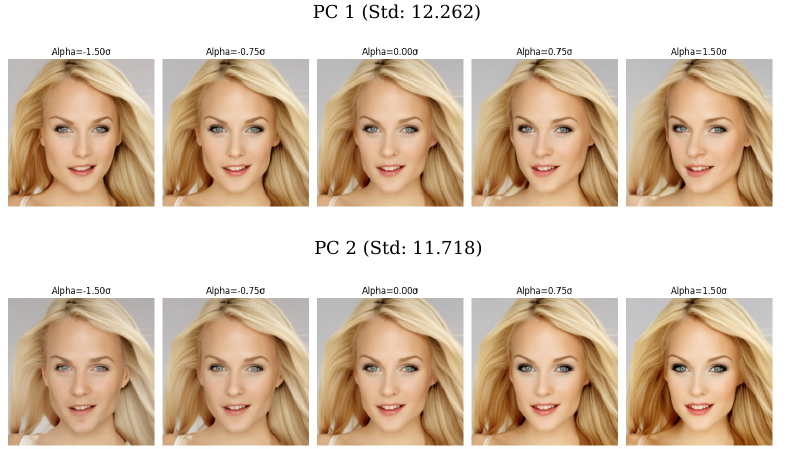}
    \caption{Principal directions of the $Z$-space. The first one corresponds to head position, the second one to gender.}
    \label{fig:interpolation_pca}
\end{figure}

These results illustrate that our token representation supports the linear and additive assumption done in Sec.\ref{sec:structural_assumption}.



\begin{table*}[h!]
    \centering
    \caption{U-Net Architecture Specifications across datasets}
    \label{tab:unet_arch_comparison}
    \begin{tabular}{cccc}
\toprule
 & \multicolumn{3}{c}{\textbf{Dataset}} \\
\cmidrule(lr){2-4}
\textbf{Parameter} & \textbf{FFHQ} & \textbf{AFHQ} & \textbf{BraTS} \\
\midrule
$f$ & 4 & 8 & 4 \\
$z$-shape & $64 \times 64 \times 3$ & $32 \times 32 \times 4$ & $32 \times 32 \times 4$ \\
$|\mathcal{Z}|$ & 8192 & 4096 & 4096 \\
Channels & 128 & 128 & 128 \\
Depth & 2 & 2 & 2 \\
Channel Multiplier & [1, 2, 4, 4] & [1, 2, 4, 4] & [1, 2, 4, 4] \\
Attention resolutions & [32, 16, 8] & [32, 16, 8] & [32, 16, 8] \\
Head Channels & 64 & 64 & 64 \\
Cross-Attention dim & 128 & 128 & 128 \\
\hline
\# Params & 160M & 160M & 160M \\
Compute & 90 GFlops & 34 GFlops & 34 GFlops \\
\hline
Prediction & $v$-pred & $v$-pred & $v$-pred \\
Training type &  OT Flow-Matching &  OT Flow-Matching &  OT Flow-Matching \\
\hline
Batch Size & 64 & 64 & 64 \\
Optimizer & AdamW & AdamW & AdamW \\
Iterations & 300k & 300k & 300k \\
Learning Rate & 1e-4 & 1e-4 & 1e-4 \\
EMA decay & 0.9999 & 0.9999 & 0.9999 \\
\bottomrule
\end{tabular}
\end{table*}

\begin{table}[h!]
    \centering
    \caption{CrossQueryEncoder Architecture Specifications.}
    \label{tab:cross_query}
    \begin{tabular}{@{}cc@{}}
    \toprule
    Output tokens count & 32 \\
    Output token dim & 128 \\
    Input embedding dim & 768 \\
    Latent dimension ($d$) & 512 \\
    Depth & 4 \\
    Attention heads & 12 \\
    Head dimension & 64 \\
    Feedforward multiplier & 4 \\
    \midrule
    \# Params &  15M \\
    Compute & 2.3 GFlops\\
    \midrule  
    Batch Size & 64 \\
    Optimizer & AdamW \\
    Learning Rate &1e-4 \\
    \bottomrule
    \end{tabular}
\end{table}

\begin{table}[h!]
\centering
\caption{ColorEncoder Architecture}
\label{tab:color_encoder}
\footnotesize
\begin{tabular}{@{}cc@{}cc@{}}
\toprule
\textbf{Layer} & \textbf{Operation} & \textbf{Kernel/Stride} & \textbf{Output Shape} \\
\midrule
Input & RGB Image & -- & $3 \times H \times W$ \\
Conv1 & Conv2D + ReLU & $5 \times 5$ / 2 & $32 \times \frac{H}{2} \times \frac{W}{2}$ \\
Conv2 & Conv2D + ReLU & $5 \times 5$ / 2 & $128 \times \frac{H}{4} \times \frac{W}{4}$ \\
Pool & AdaptiveAvgPool2D & $1 \times 1$ & $128 \times 1 \times 1$ \\
Flatten & Reshape & -- & $128$ \\
FC & Linear & -- & $128$ \\
\bottomrule
\end{tabular}
\end{table}

\subsection{Training setup}
\label{app:trainin_setup_1}
\paragraph{Datasets.}
We use \textbf{FFHQ} \cite{karras2019style} as our main training experiment dataset. Following recommendation from their creator, we use the first 60k images as a training set and the last 10k images as the testing set. Images are recized to 256$\times 256$, and the dataset is processed by the publically available VQ-VAE trained on CelebA-HQ\footnote{\url{https://huggingface.co/CompVis/ldm-celebahq-256}}. Despite FFHQ being notoriously more diverisifed than CelebA, this VQ-VAE model performed better in raw reconstuction than the standard \texttt{sd-vae-ft-mse}.

We use \textbf{CelebA-HQ} \cite{liu2015faceattributes} as a generalization dataset to assess the reconstruction quality of our first stage training. We use the 30k images as the testing set, and resize them to 256$\times 256$. The dataset is processed using the same VQ-VAE as FFHQ.

\begin{figure}
    \centering
    \includegraphics[width=0.98\linewidth]{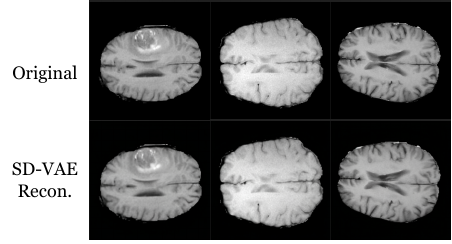}
    \caption{Visual inspection of the BraTS 2023 autoencoding using \texttt{sd-vae-ft-mse}.}
    \label{fig:brats_vae_visu}
\end{figure}
We use \textbf{BraTS 2023} \cite{kazerooni2024brain} as an auxiliary dataset with a different domain. Since we need the segmentation masks to have tumor labels, we decide to re-split the train set (since the validation set does not contain the segmentation labels). To avoid any data leakage, we split patient-wise to get 90/10 train/test splits. We use the T1-normalized MRI volume and apply z-score normalization restricted to non-background voxels using a brain mask. The volume is then sliced axially, and each slice is clipped to a fixed intensity range (-3 to 3), min–max normalized, and rescaled to a range of [0, 255]. Slices that have more than 99.5\% of black pixels are discarded, and remaining slices are replicated across three channels and saved as PNG images. Using the segmentation masks, each slice is labeled with a binary tumor-presence indicator. We use the pre-trained \texttt{sd-vae-ft-mse} as the autoencoder for this dataset. A visual inspection (Figure \ref{fig:brats_vae_visu}) confirms that despite being out-of-domain, this dataset can be easily processed with this autoencoder. The resulting training split consist of 151,824 image of 1,323 patients and the test split consist of 16'830 images of 147 patients. For evaluation, we concentrate on the slices between 85 and 100.

We use \textbf{AFHQ} \cite{choi2020starganv2} as another auxiliary dataset with a different domain than faces. The dataset is comprised of 3 main classes: \texttt{cat}; \texttt{dog}; \texttt{wild}. We resize all 16,130 images to $256\times 256$, and perform geometrical augmentations (random flip and light random crop) during first stage training to mitigate overfitting. We use the \texttt{sd-vae-ft-mse} as the autoencoder for this dataset.

\paragraph{Optimization and Training.} For efficiency, we precompute DINOv3 and VAE features. We jointly train the U-Net, Cross-Query, and Color encoder for 300k steps, using AdamW and a batch size of 32. The learning rate is warmed up from $10^{-6}$ to $10^{-4}$ for 5000 steps, then kept constant at $10^{-4}$ for the rest of the training. All models are trained in mixed precision with \texttt{bfloat16}.
Training was performed on a single NVIDIA H100 GPU, equipped with an Intel Xeon Platinum 8468 CPU with 24 active cores (workers). Training takes about 48 hours on this configuration.

\paragraph{Flow Matching.} The only loss at play in this stage is $\mathcal{L}_{FM}$ defined in Eq. \ref{eq:flow_matching_loss}. We use Minimatch Optimal Transport Flow matching \cite{tong2024improving}, where in each batch, we create data-noise pairs by minimizing the overall OT distance between the two distributions.

\paragraph{Diffusion inference details.}
We sample our images using a simple Euler integration, with 20 steps (unless otherwise noted). For all methods, the same starting noise is used. A study on the effect of random noise initialization can be found in App. \ref{app:cond_qualitative}.

\subsection{Baselines}
\paragraph{DiffAE.}
We use the official DiffAE repository\footnote{https://github.com/konpatp/diffae}, and use their pretrained weights \texttt{ffhq\_256\_autoenc}. The original DiffAE method heavily relies on noise inversion to perfom the edits. However, we saw in Sec. \ref{sec:theory} that to properly control target attributes, we need these attributes to \textit{not} depend on initial noise. Therefore, we generate the images from random noise. This explains most of the performance metric differences with the original paper.

\paragraph{EncDiff.}
The EncDiff architecture differs with our work mainly on how the conditioning tokens are constructed. Therefore, for a fair comparison, we re-empleted the exact EncDiff encoder, and trained it in a similar setting than our DINO-based Cross-Query extractor. This ensures that the difference in reconstruction is due to \textit{conditioning quality}, and not diffusion training differences. We trained EncDiff to produce exactly the same number of tokens than us, resulting in a much tighter bottleneck.

\paragraph{DINOv3 and CLIP.}
For efficiency, we first pre-compute DINOv3 and CLIP embeddings and store them. For CLIP, images were resized to $224\times 224$ for compatibility. All of the last activations were used to condition the diffusion model.
\begin{itemize}[leftmargin=*,topsep=2pt,itemsep=1pt]
    \item For DINOv3: this includes the \textit{CLS} token, the register tokens, and the 256 encoded patch tokens, resulting in $K=261$ tokens.
    \item for CLIP: this includes the \textit{CLS} token and the 197 patch tokens, resulting in $K=198$ tokens.
\end{itemize}
All of the tokens were linearly projected in $d=128$ for consistency across different backbones (except for pre-trained DiffAE which does not condition using cross-attention, and uses $d=512$).
\begin{figure}[t]
    \centering
    \includegraphics[width=0.99\linewidth]{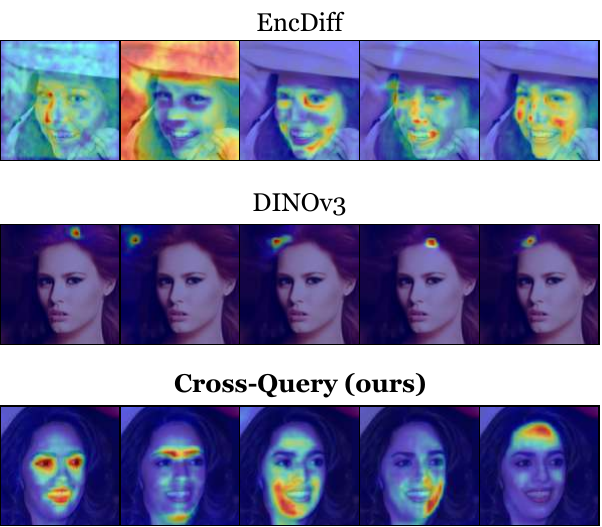}
    \caption{Cross-attention maps for some conditioning tokens, comparing our conditioning to raw DINOv3 and EncDiff. Our tokens attend to large, semantically coherent regions (e.g., facial parts), whereas DINO tokens are overly local and EncDiff tokens diffuse, supporting the claim that our conditioning yields a more structured, interpretable representation.}
    \label{fig:attn_comparison_icml}
\end{figure}
\section{Additional Results for the Conditioning Tokens.}

\subsection{Additional Qualitative Results for Conditioning}
\label{app:cond_qualitative}

 \begin{figure*}[t]
    \centering
    \includegraphics[width=0.99\linewidth]{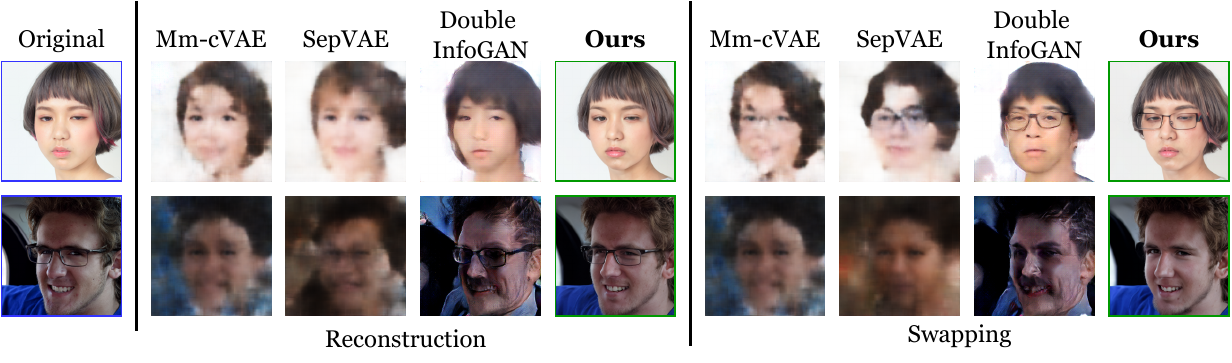}
    \caption{Comparing reconstruction and swapping between Diff-CA (ours) and CA baselines. The baselines succeed in swapping the salient factor, but only generate blurry pictures, undermining the fidelity to the original image. }
    \label{fig:swapping_baselines}
\end{figure*}

\paragraph{Cross-Attention maps.} To understand what kind of spatial or semantic information each conditioning token focuses on during the generation process, we analyze the cross-attention maps within the U-Net. Specifically, we extract the attention matrices, $\text{Softmax}(QK^T/\sqrt{d_k})$, from the cross-attention layers in the final decoder block of the U-Net where our concept tokens interact with the spatial features of the U-Net. These maps are averaged across all sampling timesteps to provide a stable visualization of each token's influence.

Visualizations (Figure \ref{fig:attn_comparison_icml}) show that our cross-query concept tokens produce attention maps that highlight broad semantic regions (e.g., "eyes and mouth", "lower-left face", "eyebrows and chin"). In contrast, raw DINOv3 attention maps tend to be highly localized and correspond to specific image patches. For EncDiff, the attention maps appear more diffuse or less structurally coherent, indicating a less specialized encoding by its tokens.

\paragraph{Cross-Attention maps.} To understand what kind of spatial or semantic information each conditioning token focuses on during the generation process, we analyze the cross-attention maps within the U-Net. Specifically, we extract the attention matrices, $\text{Softmax}(QK^T/\sqrt{d_k})$, from the cross-attention layers in the final decoder block of the U-Net where our concept tokens interact with the spatial features of the U-Net. These maps are averaged across all sampling timesteps to provide a stable visualization of each token's influence.

Visualizations (Figure \ref{fig:attn_comparison_icml}) show that our cross-query concept tokens produce attention maps that highlight broad semantic regions (e.g., "eyes and mouth", "lower-left face", "eyebrows and chin"). In contrast, raw DINOv3 attention maps tend to be highly localized and correspond to specific image patches. For EncDiff, the attention maps appear more diffuse or less structurally coherent, indicating a less specialized encoding by its tokens.

\paragraph{Dependence on starting noise.}
Fig. \ref{fig:starting_noise} shows  compares the dependence on the starting noise of DiffAE vs. our CrossQuery conditioning. Our method achieves better structural reconstruction across various starting noises. This seems to indicate that DiffAE's reconstruction power comes, in part, from the \textit{noise inversion} protocol. Therefore, precise structural edits are not possible with DiffAE by acting solely on the conditioning.

\begin{figure}[h]
    \centering
    \includegraphics[width=0.98\linewidth]{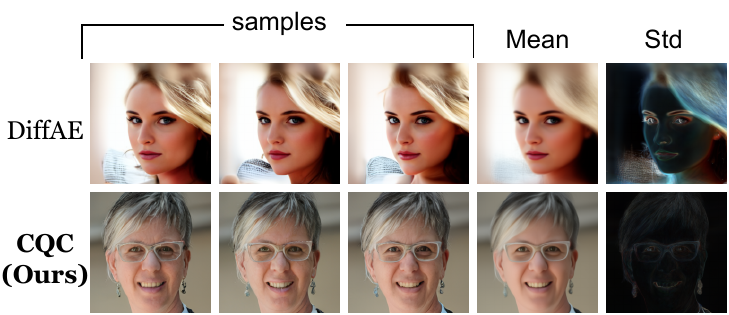}
    \caption{Dependence of the generated images on different starting noises. }
    \label{fig:starting_noise}
\end{figure}

\paragraph{Effect of the color token.} We display the information learned by the color token in Fig.~\ref{fig:color_token_ablation}.
\begin{figure}[t]
    \centering

    \includegraphics[width=0.95\linewidth]{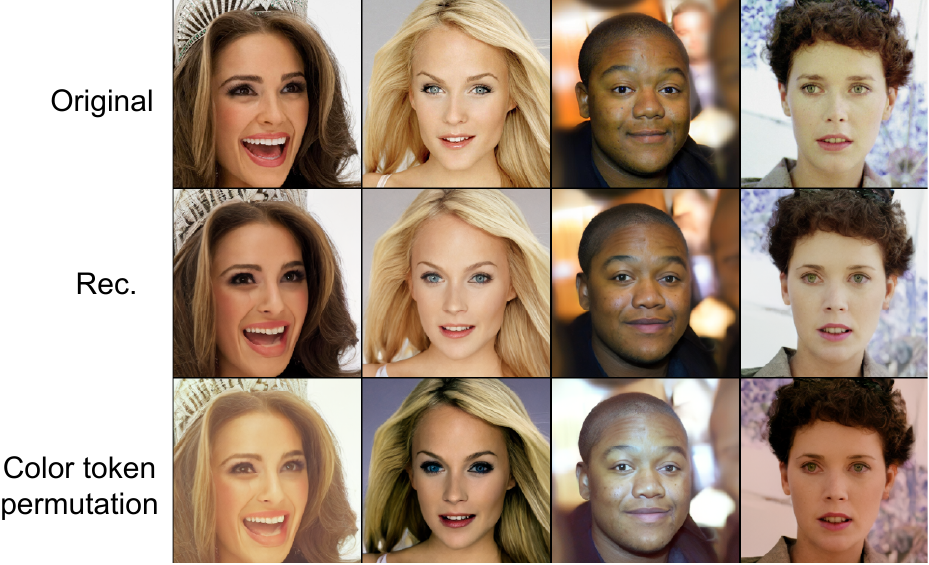}
    \caption{Effect of the color token. \textbf{Top row}: original images. \textbf{Middle row}: reconstruction using $T$ and $t_{col}$. \textbf{Bottom row}: we permute the color tokens $t_{col}$ between the 4 images and reconstruct from this modified conditioning. The color token $t_{col}$ contains color histogram information.}
    \label{fig:color_token_ablation}
\end{figure}

\subsection{Additional Quantitative Results for Conditioning}
\label{app:cond_quant}

\paragraph{ID-Sim implementation details.} The ID-Sim metrics correspond to the cosine similarity in a Face Recognition model's latent space between the original and the generated images. It aims at verifying that the generated face represents the same person as the original image. We use the latent space of an IResnet 50 \cite{he2016deep} trained on MSMV3 \cite{msmv3} with the ArcFace loss \cite{arcface_paper}. 

More precisely, we extract $Id_o$ and $Id_g$, the latent code for the original, and generate images which naturally lie on the unit hypersphere $S^{512}$. We then derive ID-Sim as follows: 
\begin{equation}
    \text{ID-Sim} = cos(Id_o, Id_g) = Id_o^T Id_g
\end{equation}

\paragraph{Reconstruction.} We present additional reconstruction metrics of the models trained on FFHQ and tested on CeleA-HQ with various ablations in Tab. \ref{tab:ffhq_comparison}. As we see, DINOv3 leads to the best reconstruction metrics due to its massive conditioning size, combined with DINO-based training which favors very structured outputs. There are promising existing approaches that use these DINO feature maps as the direct input to diffusion models \cite{zheng2025diffusion}. We instead focus on using these tokens  to condition diffusion through cross-attention.

Our cross query extractor module, despite learning to a much smaller conditioning space, still achieves competitive results, particularly in perceptual metrics such as LPIPS and DISTS. This seems to indicate that although very fine-graine textures (such as hair or grass) can be lost through the distilation process, most of the semantic information is still present in the $Z$-space.

We present reconstruction metrics of our model on the other datasets in Tab \ref{tab:reoncstruciton_datasets}, and study the effect of sampling steps in Tab.~\ref{tab:ablation_ffhq_cqc_steps}.

\begin{table*}[t]
    \centering
    \resizebox{\textwidth}{!}{
    \begin{tabular}{ccccccccccc}
    \toprule
       Encoder  &  Steps & Cond. Size  & PSNR$\uparrow$ &SSIM$\uparrow$ & LPIPS$\downarrow$ & DISTS$\downarrow$ & FID$\downarrow$ & FD-DINO$\downarrow$ & KID$\downarrow$ &ID-Sim$\uparrow$\\ \midrule
       DiffAE &20&512&19.09&0.530&0.209&0.173&14.28&32.89&0.013 &  0.407\\
       DiffAE &200&512& 18.81 & 0.507 & 0.196 & 0.161 & 4.76 & 28.70 & 0.003& 0.404 \\
       EncDiff & 20 &$32 \times 128$ & 17.03& 0.420 & 0.323  & 0.210  & 18.43 & 64.91 & 0.014& 0.313 \\
       CLIP & 20 & $198\times 128$&18.23 & 0.490 & 0.201&0.169 & 7.69 & 31.45& 0.006& 0.538 \\
       DINOv3 &20 &  $261 \times 128$ &23.31& 0.645 &  0.099& 0.114 & 5.66 &20.89 & 0.005& 0.609\\
       \midrule
       \textbf{CQ} &20 & $32 \times 128$ & 21.76 & 0.589 &  0.124& 0.131 & 5.59 & 22.39 & 0.005&0.515 \\
       \textbf{CQC} & 20 &  $32 \times 128$ & 22.57 & 0.611 & 0.114 & 0.118 & 5.52 & 22.28 & 0.005&  0.523\\
       \bottomrule
    \end{tabular}
    }
    \caption{Additional reconstruction results across various setups. Training is done on the train split of FFHQ and evaluated on CelebA-HQ in $256\times 256$. Sampling is performed from the same random noise. \textbf{CQ}: Cross-Query without color token; \textbf{CQC}: Cross-Query with Color token.}
    \label{tab:ffhq_comparison}
\end{table*}

\begin{table*}[t]
    \centering
    \resizebox{\textwidth}{!}{
    \begin{tabular}{cccccccccc}
    \toprule
       Encoder  &  Steps & Cond. Size  & PSNR$\uparrow$ &SSIM$\uparrow$ & LPIPS$\downarrow$ & DISTS$\downarrow$ & FID$\downarrow$ & FD-DINO$\downarrow$ & KID$\downarrow$ \\ \midrule
       \textbf{CQC} & 10 & $32 \times 128$ & 22.79 & 0.623 & 0.120 & 0.122& 8.36& 27.85 & 0.008 \\
       \textbf{CQC} & 20 &  $32 \times 128$ & 22.57 & 0.611 & 0.114 & 0.118 & 5.52 & 22.28 & 0.005 \\
       \textbf{CQC} & 50 & $32 \times 128$ & 22.41 & 0.603 & 0.113& 0.118& 4.42& 19.88 & 0.003  \\
       \bottomrule
    \end{tabular}
    }
    \caption{Effect of the number of sampling steps on our cross query encoder reconstruction performance.}
    \label{tab:ablation_ffhq_cqc_steps}
\end{table*}

\begin{table*}[t]
    \centering
    \resizebox{\textwidth}{!}{
    \begin{tabular}{cccccccc}
    \toprule
       Dataset  &  Sampling Steps & Conditioning Size  &SSIM$\uparrow$ & LPIPS$\downarrow$ & DISTS$\downarrow$ & FID$\downarrow$ & FD-DINO$\downarrow$ \\\midrule
       FFHQ &20&$32 \times 128$ & 0.585 & 0.120 & 0.127 & 6.18 & 16.18 \\
       AFHQ & 20 & $32 \times 128$ & 0.433 & 0.155& 0.139& 8.13 &15.91\\
       BraTS & 20 & $32 \times 128$ & 0.649 & 0.151 & 0.156& 31.54& 35.09 \\
       \bottomrule
    \end{tabular}
    }
    \caption{Additional reconstruction metrics of our CrossQuery+Color Encoder on different datasets.}
    \label{tab:reoncstruciton_datasets}
\end{table*}

\section{Implementation Details for Separating Network}
\label{app:separation_details}
\subsection{Architecture}

The separator network $E_{\theta_{E}}$ is designed to process the conditioning token sequence $Z$ and perform a structured decomposition into salient $\hat{Z}_{S}$ and common $\hat{Z}_{C}$ components.
\paragraph{Transformer Backbone.} 
We implement $E_{\theta_{E}}$ as a 5-layer Transformer encoder. Each layer consists of a multi-head self-attention (MHSA) block with 12 heads and a feed-forward network (FFN) with a hidden dimension of 512, following the standard Transformer architecture. The input sequence $Z$ consists of $K=32$ tokens of dimension $d=128$. No positional encoding is needed, since the conditioning DINOv3 tokens were already processed with their own positional encoding.

\paragraph{Common-Salient Head.}
Following the setup in Fig.~\ref{fig:combined_pipeline_and_query}, we prepend a learnable \texttt{CLS} token to the input sequence $Z$. After processing through the Transformer blocks, the state of the \texttt{CLS} token is passed through a linear projection layer to produce the salient code $\hat{Z}_{S} \in \mathbb{R}^{d}$. The remaining $K$ output tokens are treated as the common code $\hat{Z}_{C} \in \mathbb{R}^{K \times d}$. This design ensures that the salient information is distilled into a compact representation while the common information preserves spatial and semantic context across the original token sequence length.

\paragraph{Additive Constraint.}
The final conditioning fed to the generator $G_{\psi}$ is reconstructed as $\hat{Z} = \hat{Z}_{S} + \hat{Z}_{C}$. This additive structure is strictly enforced by the reconstruction loss $\mathcal{L}_{rec}$.

\subsection{Training setup}
\label{app:training_setup_2}
\paragraph{Weak Supervision and Binarization.}
Our framework relies on weak binary supervision $Y \in \{0, 1\}$ to separate common from salient factors. Since many benchmark datasets contain continuous or multi-class labels, we apply the following binarization protocols for training:

\begin{itemize}
    \item \textbf{FFHQ (Glasses):} We define the background ($Y=0$) as images with the "No Glasses" label and the target ($Y=1$) as any image containing eyewear. For evaluation, we utilize held-out fine-grained labels (Reading Glasses, Sunglasses) to assess latent structure discovery.
    \item \textbf{FFHQ (Headpose):} Samples are binarized based on the yaw angle $\theta$. We define $Y=0$ (Neutral) if $|\theta| \le 8^\circ$ and $Y=1$ (Turned) if $|\theta| > 12^\circ$. This gap ensures a clear separation between distributions during training. Evaluation is performed by measuring the Mean Absolute Error (MAE) of yaw predicted by a separate regressor on swapped images.
    \item \textbf{FFHQ (Smile):} We apply a threshold of $0.5$ on the smiling attribute. We acknowledge significant label noise in this attribute and thus restrict our analysis to qualitative results.
    \item \textbf{AFHQ (Species):} We utilize a binary Cat vs. Dog setup for training to treat species-specific traits as the salient factor.
    \item \textbf{BraTS 2023 (Tumor):} To ensure high-quality training signals in medical data, we restrict our volume to axial slices between indices 84 and 100. A slice is labeled $Y=1$ (Positive) if the segmentation mask contains $>500$ tumorous pixels; otherwise, it is $Y=0$.
\end{itemize}

\paragraph{Preprocessing.}
We measure reconstruction and swapping accuracies during training. For efficiency, we evaluate the accuracy using a classifier trained on the conditioning space. 

\paragraph{Adversarial training.}
Instead of relying on a classical adversarial setup using GAN-like adversarial losses, we opt for a Gradient Reversal Layer training procedure:
\begin{itemize}
    \item \textbf{Forward pass}: The encoder $E_{\theta_E}$ produces samples $(\hat{Z}_S, \hat{Z}_C)$. A small (2-layer MLP) adversary $D_{avd}$ takes $\hat{Z}_C$ and aims at predicting the label $\hat{y}$ associated with this sample, and a standard binary cross-entropy loss (BCE) is computed between $y$ and $\hat{y}$.
    \item \textbf{Backward pass}: when gradients flow back from $D_{adv}$ to $E_{\theta_{E}}$, the gradients are multiplied by $-\lambda_{adv}$
\end{itemize}
The negative weighting of $\lambda_{adv}$ effectively allows the gradients that flow back to the encoder to \textit{maximize} the BCE loss outputted by $D_{adv}$, effectively emulating an adversarial training, with only one scalar $\lambda_{adv}$ to tune. Instead of relying on a single value of $\lambda_{adv}$, we opt for a \textbf{self tuning adversarial mechansim}, described thereafter.
\begin{figure}
    \centering
    \includegraphics[width=0.98\linewidth]{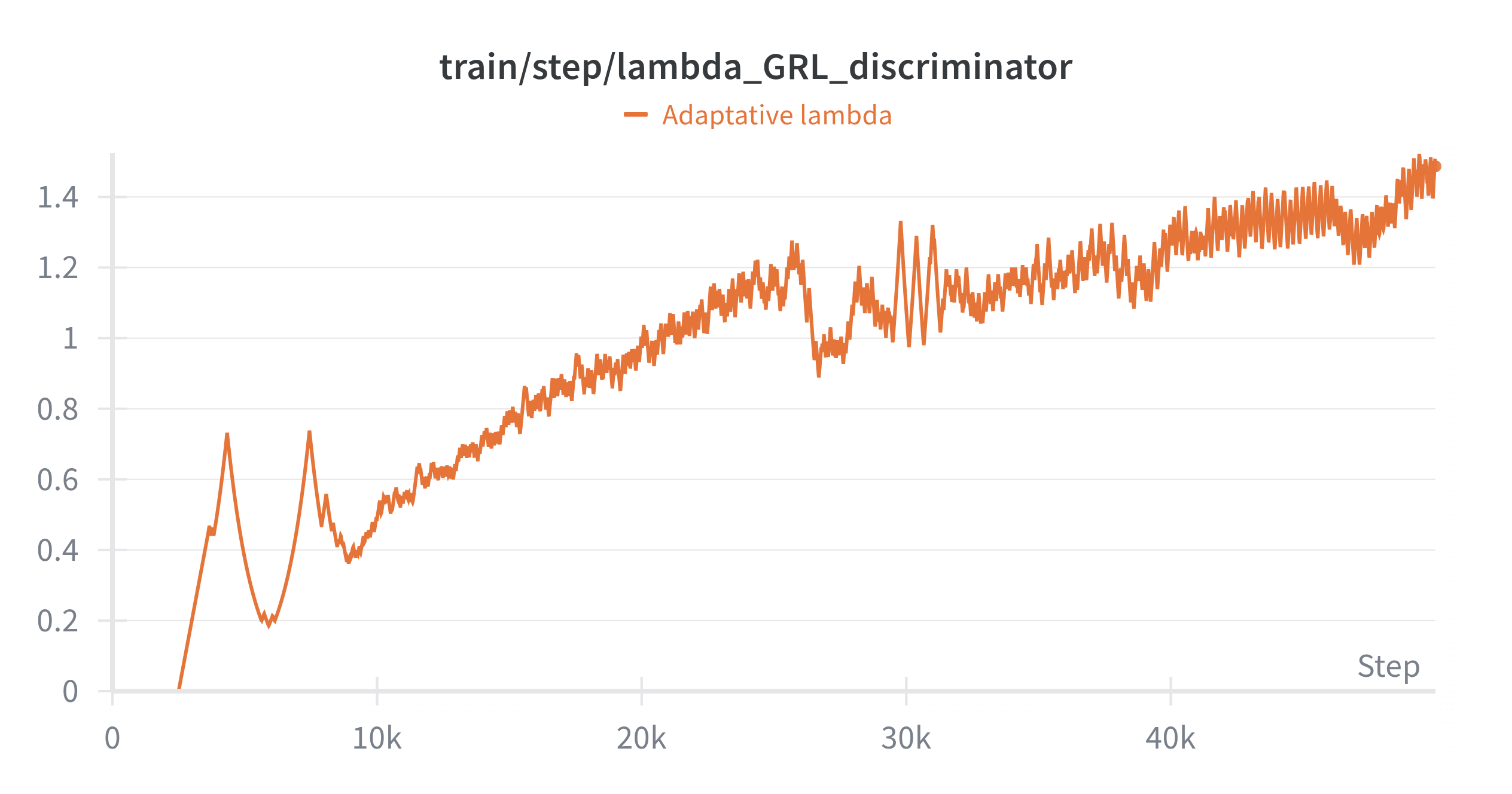}
    \caption{Training dynamics of the self-tuning GRL parameter $\lambda_{adv}$. Its goal is to dynamically maintain the discriminator accuracy at 0.80.}
    \label{fig:lambda_grl}
\end{figure}
\begin{figure}
    \centering
    \includegraphics[width=0.98\linewidth]{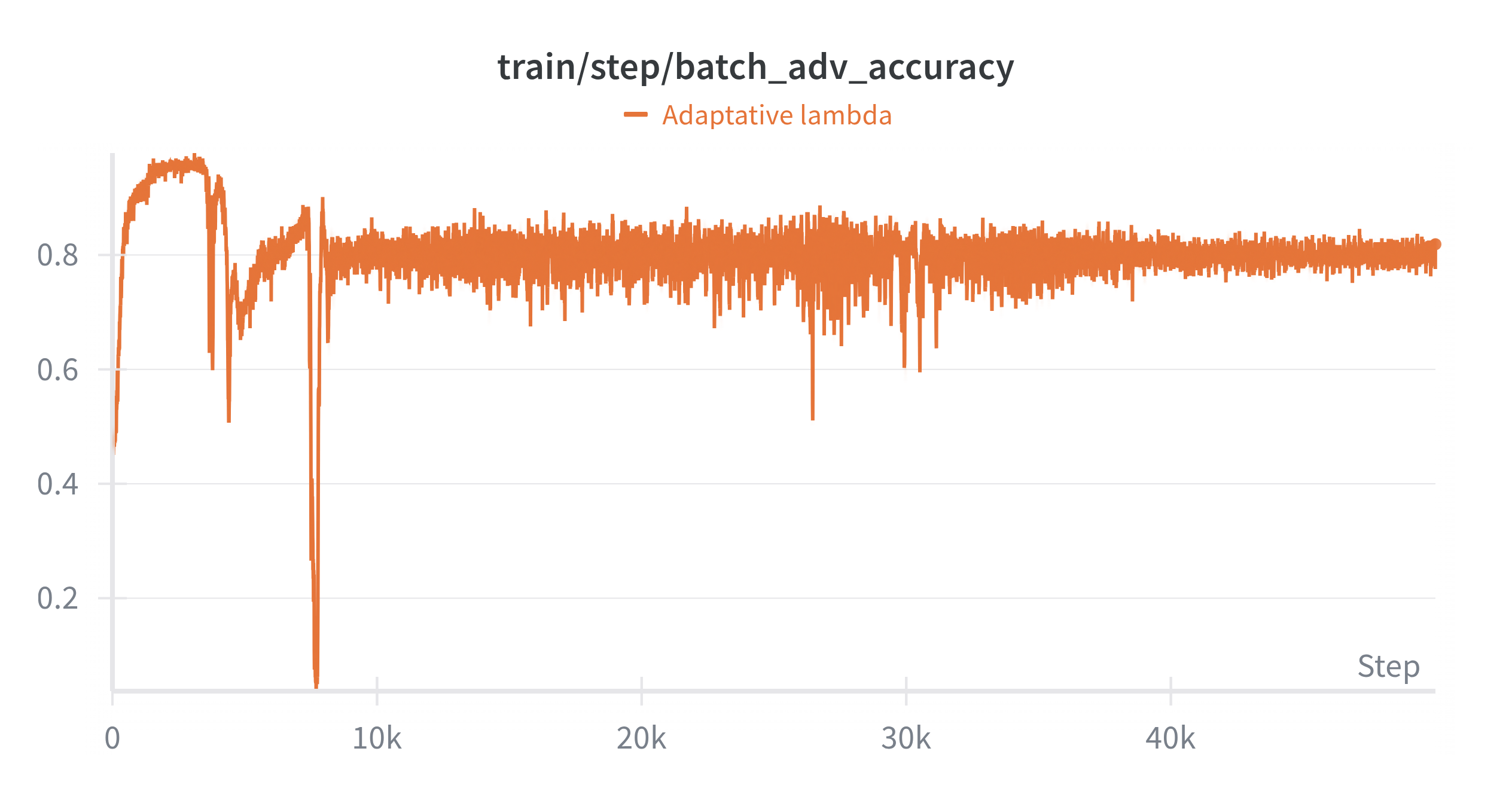}
    \caption{Training adversarial accuracy using our self-tuning GRL schedule. After a chaotic warmup phase, accuracy stabilizes at 0.80.}
    \label{fig:lambda_grl_acc}
\end{figure}
\paragraph{Self-tuning adversary.}
At each training step, a rolling accuracy $Acc$ is computed on the current batch (with a small EMA decay of 0.95). The goal of the adversary is to maximize the accuracy, while the goal of the encoder is to minimize it. During various stages of training, the power of both networks to achieve their respective tasks can vary greatly. Therefore, we establish a target accuracy $Acc_{target}$. At each step:
\begin{itemize}
    \item If $Acc > Acc_{target}$: The adversary is too strong, and we multiply $\lambda_{adv}$ by 0.999.
    \item If $Acc < Acc_{target}$: The encoder is too strong, and we devide $\lambda_{adv}$ by 0.999.
\end{itemize}
This basic controller for $\lambda_{adv}$ was sufficient to keep a smooth training for target accuracies $Acc_{target}$ between 0.70 and 0.80, with $\lambda_{adv}$ varying between $0$ and $5$.
We display an example of the evolution of $\lambda_{adv}$ in Fig.~\ref{fig:lambda_grl} and the associated accuracy evolution in Fig.~\ref{fig:lambda_grl_acc}.
\begin{table*}[t]
\centering
\caption{Detailed Swapping Results on FFHQ (Glasses Attribute). We evaluate three metrics across all six possible attribute transfers between No Glasses (N), Reading Glasses (R), and Sunglasses (S). Models were trained only on binary supervision (N=0; R=1, S=1) with 60,000 samples. Evaluation is performed with 1,000 swapped pairs for each fine-grained class.}
\label{tab:full_swapping_results}
\begin{small}
\begin{tabularx}{\textwidth}{l l @{\extracolsep{\fill}} cccccc | c}
\toprule
& & \multicolumn{2}{c}{\textbf{From No Glasses}} & \multicolumn{2}{c}{\textbf{From Reading}} & \multicolumn{2}{c}{\textbf{From Sun}} & \textbf{Aggregate} \\
\cmidrule(lr){3-4} \cmidrule(lr){5-6} \cmidrule(lr){7-8} \cmidrule(lr){9-9}
\textbf{Method} & \textbf{Metric} & $N \to R$ & $N \to S$ & $R \to N$ & $R \to S$ & $S \to N$ & $S \to R$ & \textbf{Score} \\
\midrule
MM-cVAE & Acc $\uparrow$ & 15.6 & 27.1 & \textbf{1.00} & 26.0 & \textbf{1.00} & 0.00 & 11.78 \\
 & ID-Sim\textsubscript{rec} $\uparrow$ & \textbf{.592} & .394 & .591 & .465 & .390 & .390 & .470 \\
& ID-Sim\textsubscript{real} $\uparrow$ & .302 & .298 & .302 & .309 & .300 & .300 & .302 \\
\midrule
SepVAE & Acc $\uparrow$ & 31.1 & 39.4 & \textbf{1.00} & 42.3 & \textbf{1.00} & 0.00 & 19.58 \\
 & ID-Sim\textsubscript{real} $\uparrow$ &  .300 & .298 & .300 &  .306 & .301 & .301 & .301  \\
\midrule
D.InfoGAN & Acc $\uparrow$ & 87.8 & 47.5 & 99.8 & 46.1 & \textbf{1.00} & 0.00 & 47.03 \\
 & ID-Sim\textsubscript{rec} $\uparrow$ & .412 & .415 & .424 & .425 & .422 & .422 & .413 \\
& ID-Sim\textsubscript{real} $\uparrow$ & .311 & .310 & .318 & .320 & .312 & .312 & .314  \\
\midrule
 \textbf{Diff-CA} & Acc $\uparrow$ & \textbf{91.4} & \textbf{94.2} & 98.8 & \textbf{91.5} & 98.7 & \textbf{92.9} & \textbf{94.48}\textsuperscript{\textcolor{green}{(+47.18)}} \\
 \textbf{(Ours)} & ID-Sim\textsubscript{rec} $\uparrow$ &  0.507 & \textbf{.458} & \textbf{.510} & \textbf{.492} & \textbf{.427} & \textbf{.453} & \textbf{.474}\textsuperscript{\textcolor{green}{(+.004)}} \\
 & ID-Sim\textsubscript{real} $\uparrow$ &  \textbf{0.468} & \textbf{.431} & \textbf{.463} & \textbf{.452} & \textbf{.398} & \textbf{.416} & \textbf{.438}\textsuperscript{\textcolor{green}{(+.124)}} \\
\bottomrule
\end{tabularx}
\end{small}
\begin{flushleft}
\footnotesize{Notes: \textbf{Acc} denotes top-1 swapping accuracy of a fine-grained classifier. \textbf{ID-Sim} is the ArcFace cosine similarity between non-swapped and swapped identities: we compute it between the swappings and the non-swapped reconstructions of the model, ID-Sim\textsubscript{rec};and between the swapping and the original images, ID-Sim\textsubscript{real}. \textbf{KID} is the Kernel Inception Distance. }
\end{flushleft}
\end{table*}
\begin{figure}
    \centering
    \includegraphics[width=0.98\linewidth]{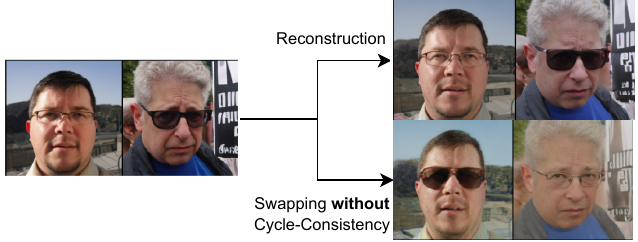}
    \caption{Example of swapping with a model trained \textbf{without} cycle consistency. Swapping $\hat{Z}_S$ leads to major color changes from the reconstruction, indicating a hidden dependance between $\hat{Z}_C$ and $\hat{Z}_C$}
    \label{fig:without_cycle}
\end{figure}
\paragraph{Swaping Cycle-Consistency.}
When trained only using reconstruction loss $\mathcal{L}_{rec}$, adversarial loss $\mathcal{L}_{adv}$, and pinning loss $\mathcal{L}_{pin}$, we found that the model had good reconstruction performances, and that the salient $\hat{Z}_S$ had correct classification accuracy. However, swapping often lead to unexpected changes in the images, such as color shifts and other attributes inconsistencies (see Fig.~\ref{fig:without_cycle}). We hypothesize this happens because $\mathcal{L}_{adv}$ only enforces $\hat{Z}_C \indep Y$ and not $\hat{Z}_C \indep \hat{Z_S}$, which a somewhat correct but imperfect approximation. Therefore, in this case conditions for Th. \ref{sec:identifiability} are not met, and we observe a "leakage" from context to salient. 

Cycle consistency loss aims to correct this. Intuitively, we want swapped counterfactuals (samples of the form ${Z_{mix} =\hat{Z}_S^a + \hat{Z}_C^b}$ with $\hat{Z}_S^a$ and $\hat{Z}_C^b$ from different images $a$ and $b$) to be "valid" samples. In other terms, intuitively we want:
\begin{equation}
\label{eq:counter_equality}
\tag{H\textsubscript{mix}}
Z_{mix} \overset{Law}{=} Z
\end{equation}
which is the modeling assumption used in Sec.~\ref{app:cycle}. Importantly, this condition is not enforced by an explicit loss, but rather serves as an intuitive guideline motivating the design of the cycle-consistency objective.

While one could in principle enforce~\eqref{eq:counter_equality} using an adversarial discriminator (similarly to how $\mathcal{L}_{adv}$ enforces $\hat{Z}_C \indep Y$), we found such approaches to be unstable in practice. Instead, we rely on the separator $E_{\theta_E}$ itself to act as a weak, implicit discriminator. Specifically, we maintain an exponential moving average (EMA) copy of $E_{\theta_E}$, which is used to re-encode swapped counterfactuals $Z_{\mathrm{mix}}$. If $Z_{\mathrm{mix}}$ does not resemble a valid sample from the true conditioning distribution, the frozen EMA encoder fails to satisfy the cycle-consistency criterion. This mechanism limits co-adaptation between the representation $(\hat{Z}_S, \hat{Z}_C)$ and the separator, and should be understood as a heuristic regularization strategy rather than a formal guarantee of distributional matching or independence. This design is conceptually related to self-distillation methods such as BYOL~\cite{grill2020bootstrap}, where an EMA target network prevents collapse toward trivial solutions.

\paragraph{Optimization and Hyperparameters.}
We jointly optimize the separator and adversarial discriminator for 1000 epochs using the AdamW optimizer with a batch size of 1024. Training was performed on a single NVIDIA A100 GPU, equipped with an AMD EPYC 7543 processor with 16 active cores (workers). Training takes about 3-4 hours on this configuration.

For the pinning loss $\mathcal{L}_{pin}$, we use a weighting of 1 on the negative classes (to force all background salients to colapse at 0), and a weighting of 0.1 on the positive classes. The cycle consistency loss weighting was set at 1.

\begin{table}[h!]
    \centering
    \caption{Separator network specifications.}
    \label{tab:separator_specifications}
    \begin{tabular}{@{}cc@{}}
    \toprule
    Layers & 5 \\
    Output token dim & 128 \\
    Input embedding dim & 768 \\
    Latent dimension ($d$) & 512 \\
    Depth & 4 \\
    Attention heads & 12 \\
    Head dimension & 64 \\
    Feedforward multiplier & 4 \\
    Latent initialization & Learnable queries \\
    \midrule
    \# Params &  15M \\
    Compute & 2.3 GFlops\\
    \midrule  
    Batch Size & 64 \\
    Optimizer & AdamW \\
    Learning Rate &1e-4 \\
    \bottomrule
    \end{tabular}
\end{table}

\section{Additional Results for Common-Salient Editing}
\paragraph{Detailed performance for glasses.}
\begin{table}[ht]
\centering
\caption{Detailed Head Pose Swapping Results. Accuracy measures correct orientation (Left, Straight, Right) identification, while MAE measures precision in degrees. Trained only on binary Neutral ($y=0$) vs. Turned ($y=1$) supervision.}
\label{tab:headpose_metrics}

\begin{tabular}{lcc}
\toprule
\textbf{Swap Operation} & \textbf{Accuracy} $\uparrow$ & \textbf{Yaw MAE ($^\circ$)} $\downarrow$ \\
\midrule
Add Left ($N \to L$) & 84.5 & 5.5807 \\
Add Right ($N \to R$) & 84.0 & 5.6945 \\
Remove Left ($L \to N$) & 98.0 & 3.0699 \\
Remove Right ($R \to N$) & 98.2 & 2.8752 \\
Right $\to$ Left & 88.7 & 5.5692 \\
Left $\to$ Right & 83.5 & 5.7251 \\
\midrule
\textbf{Macro-Average} & \textbf{0.8948} & \textbf{4.7524} \\
\bottomrule
\end{tabular}
\end{table}

We present the detailed swapping performance, in terms of swapping accuracy and identity perseveration in Tab.~\ref{tab:full_swapping_results}. Diff-CA achieves substantial better results in all swappings cases.

\paragraph{Detailed performance on Head Position.} We train a small regressor MLP to predict the \texttt{yaw} attribute from our images. This classifier achieves a MAE of 2,55° . We train our separator network on the binary \texttt{Neutral / Non-Neutral} attribute, then evaluate it by performing swaps between the fine-grained labels \texttt{Neutral / Left /Right}. We then evaluate classification accuracy of the swap, both in terms of discrete labels as well as with the predicted yaw value. Results are reported in Tab.~\ref{tab:headpose_metrics}.

\paragraph{PCA and UMAP of the learned salient factor.} Fig.~\ref{fig:projection} displays projections of the learned $\hat{Z}_S$. As we see, dispite being learned solely using binary labels (Glasses vs. No Glasses), the space of $\hat{Z}_S$ is organized according to finer-grained sublabels. This can be seen as an instance of concept discovery using CA.

\clearpage
\begin{figure*}
    \centering
\includegraphics[width=0.95\linewidth]{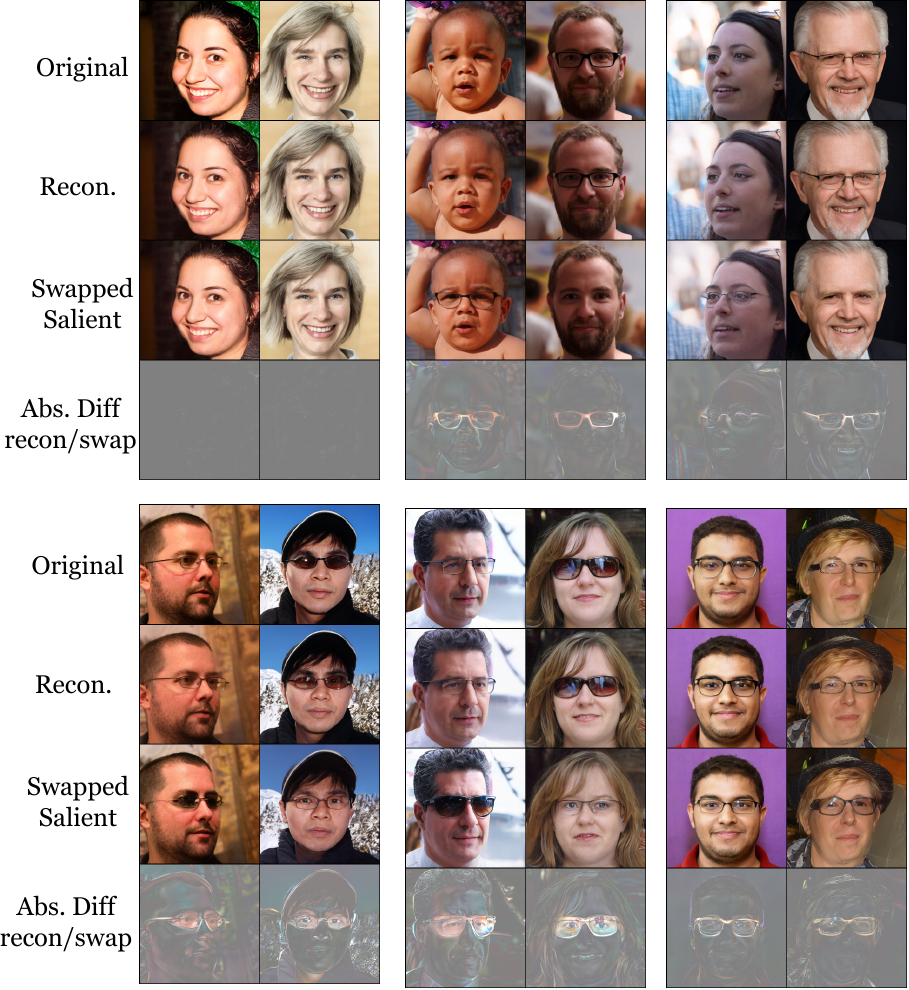}
    \caption{Additional swapping results on glasses in FFHQ using Diff-CA.}
    \label{fig:app_swap_glasses}

\end{figure*}

\begin{figure*}
    \centering
\includegraphics[width=0.95\linewidth]{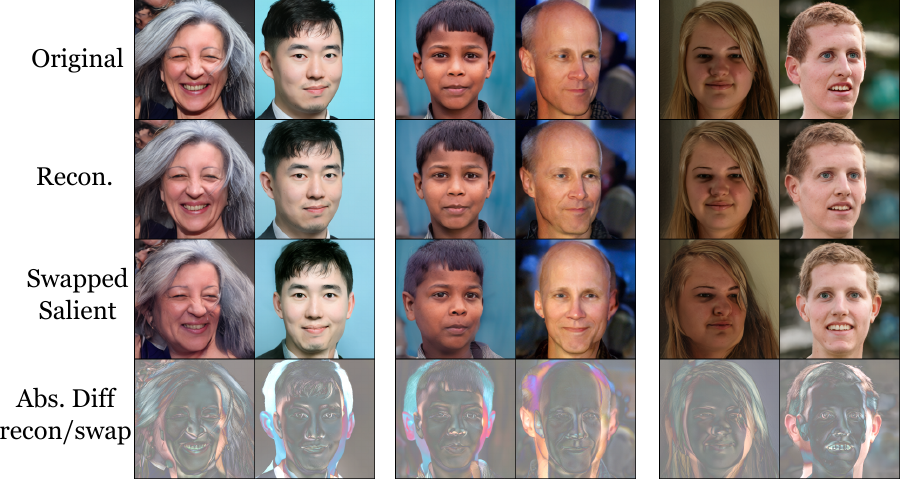}
    \caption{Additional swapping results on the head position in FFHQ using Diff-CA.}
    \label{fig:app_swap_pose}

\end{figure*}

\begin{figure*}
    \centering
\includegraphics[width=0.95\linewidth]{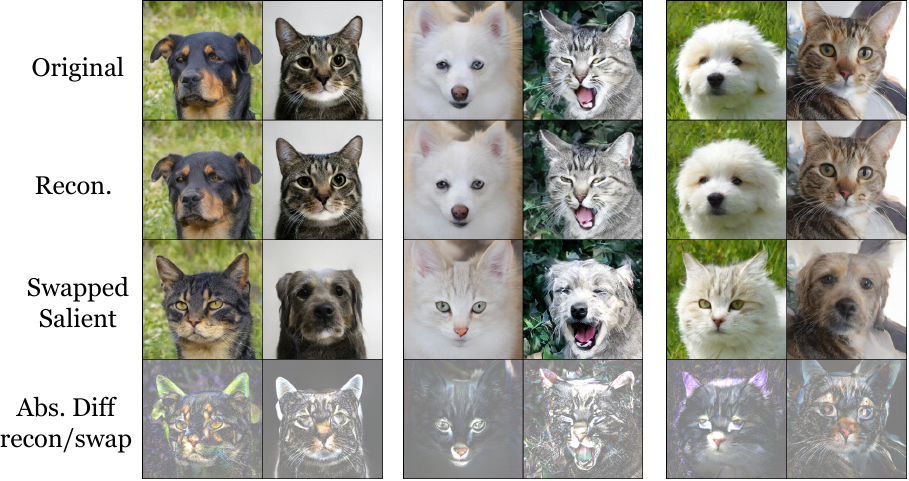}
    \caption{Additional swapping results on the cats/togs classes in AFHQ using Diff-CA.}
    \label{fig:app_swap_afhq}

\end{figure*}

\begin{figure*}
    \centering
\includegraphics[width=0.95\linewidth]{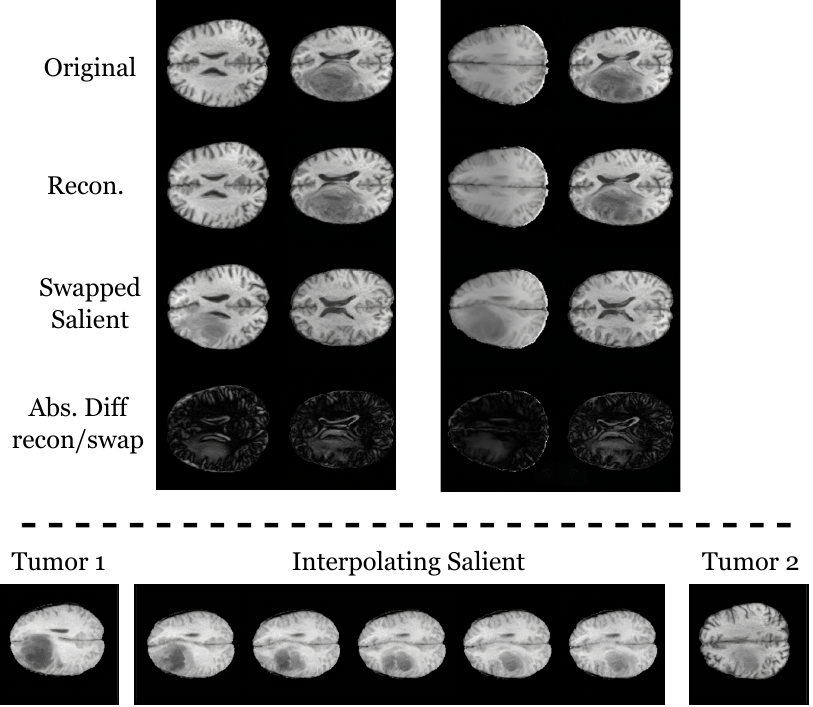}
    \caption{Additional swapping and interpolation results on the BraTS 2023 dataset using Diff-CA.}
    \label{fig:app_swap_brats}

\end{figure*}






\end{document}